\documentclass{style/ecai2025/ecai}

\usepackage[T1]{ fontenc }
\usepackage{times}
\usepackage[dvipsnames]{ xcolor }
\usepackage[normalem]{ ulem }
\usepackage[pdftex]{ graphicx }
\usepackage{soul}
\usepackage{url}
\usepackage[inline]{ enumitem }
\usepackage{ booktabs }
\usepackage{ siunitx }
\usepackage{ optidef }
\usepackage{ multirow }
\usepackage[english]{babel}
\usepackage{ algorithm, algorithmic }
\usepackage[small]{caption}
\usepackage{subcaption}
\usepackage{ pythonhighlight }
\usepackage{makecell}
\usepackage{url}
\usepackage{adjustbox}
\usepackage{microtype}
\usepackage{float}

\usepackage{ ifthen }
\usepackage{ amsmath, amssymb, dsfont, mathtools, amsthm, amsfonts, nicefrac, microtype, bm }
\usepackage{ xspace }
\usepackage[makeroom]{cancel}

\makeatletter
\DeclareRobustCommand\onedot{\futurelet\@let@token\@onedot}
\def\@onedot{\ifx\@let@token.\else.\null\fi\xspace}

\def\eg{\emph{e.g}\onedot, } 
\def\ie{\emph{i.e}\onedot, }

\makeatother

\newcommand{\partopic}[1]{\paragraph{#1}}

\newcommand{\meanstd}[2]{#1\tiny{\(\pm\)#2}}

\newcommand{\mysection}[2][]{
    \ifthenelse{ \equal{#1}{} }
    {\section{\texorpdfstring{#2}{#2}}}
    {\section{\texorpdfstring{#2}{#1}}}
}
\newcommand{\mysubsection}[2][]{
    \ifthenelse{ \equal{#1}{} }
    {\subsection{\texorpdfstring{#2}{#2}}}
    {\subsection{\texorpdfstring{#2}{#1}}}
}
\newcommand{\mysubsubsection}[2][]{
    \ifthenelse{ \equal{#1}{} }
    {\subsubsection{\texorpdfstring{#2}{#2}}}
    {\subsubsection{\texorpdfstring{#2}{#1}}}
}

\newcommand{\data}[1]{\bm{#1}}

\renewcommand{\vec}[1]{\mathbf{#1}}

\DeclareMathAlphabet{\mathsfit}{\encodingdefault}{\sfdefault}{m}{sl}
\SetMathAlphabet{\mathsfit}{bold}{\encodingdefault}{\sfdefault}{bx}{n}

\newcommand{\set}[1]{\mathbb{#1}}

\newcommand{\norm}[2]{\lVert#1\rVert_{#2}}

\newcommand{\expect}{\mathbb{E}}

\newcommand{\norml}[1]{\ell^{#1}\textrm{-norm}}

\newtheorem{proposition}{Proposition}

\NewCommandCopy{\oldexp}{\exp}
\renewcommand{\exp}[1]{\oldexp\left(#1\right)}

\newcommand{\history}{\bm{\mathcal{H}}}
\newcommand{\lossn}{L_n}
\newcommand{\losse}{L_{e}}

\newcommand{\model}{\mathcal{M}}

\usepackage[hidelinks]{ hyperref }
\usepackage[acronym]{ glossaries }
\usepackage[capitalise, noabbrev]{ cleveref }

\glsdisablehyper

\newglossaryentry{01ip}
{
    name={\(0-1\) integer program},
    description={\(0-1\) integer program is the root of EHD.}
}

\newglossaryentry{ehd}{
    name = {Explanation for MTPP},
    description={The name of the problem.}
}

\newglossaryentry{ca}
{
    name={counterfactual explanation},
    description={Counterfactual analysis is the methodology underpinning EHD.}
}

\newglossaryentry{fa}
{
    name={factual explanation},
    description={factual analysis is the methodology underpinning EHD.}
}

\newglossaryentry{ca1}
{
    name={counterfactual analysis},
    description={Counterfactual analysis is the methodology underpinning EHD.}
}

\newglossaryentry{ar}
{
    name={abductive reasoning},
    description={A general way to reason.}
}

\newglossaryentry{cr}
{
    name={counterfactual analysis},
    description={Counterfactual reasoning may be more appropriate for describing our task.}
}

\newglossaryentry{cif}
{
    name={conditional intensity function},
    description={One conditional intensity function defines an MTPP.}
}

\newglossaryentry{so}
{
    name={StackOverflow},
    description={Dataset StackOverflow}
}

\newglossaryentry{retweet}
{
    name={Retweet},
    description={Dataset Retweet}
}

\newglossaryentry{yelp}
{
    name={Yelp},
    description={Dataset Yelp}
}

\newacronym{ce}{counterfactual explanation}{Counterfactual Explanation}
\newacronym{mtpp}{MTPP}{Marked Temporal Point Process}
\newacronym{nmtpp}{NMTPP}{Neural Marked Temporal Point Process}
\newacronym{model}{CFF}{Counterfactual and Factual Explainer for MTPP}
\newacronym{mip}{MIP}{Mixed-Integer Programming}
\newacronym{co}{CO}{Combinatorial Optimization}
\newacronym{milp}{MILP}{Mixed-Integer Linear Programming}
\newacronym{lp}{LP}{Linear Programming}
\newacronym{st-gs}{ST-GS}{Straight-Through Gumbel-Softmax trick}
\newacronym{nll}{NLL}{Negative Log-Likelihood}
\newacronym{cfe}{CFE}{Counterfactual Explanations}
\newacronym{dppl-diff}{EHED}{Explainable History Events Distinguishment}
\newacronym{dppl-diff-metric}{DS}{Distinguishment Significance}
\newacronym{card-diff}{EMHD}{Explainable Minimal History Distillation}
\newacronym{card-diff-metric}{LMHS}{Length of Minimal Historical Subset}
\newacronym{rd}{RR}{Random Removal}
\newacronym{gs}{GS}{Greedy Search}
\newacronym{asd}{ASD}{Auxiliary Selection Distribution}
\newacronym{gst}{GS trick}{Gumbel-softmax trick}
\newacronym{mgst}{Mixture GS trick}{Mixture Gumbel-Softmax trick}
\newacronym{gp}{GP}{Gradient Passthrough}
\newacronym{ite}{ITE}{Individual Treatment Effect}
\newacronym{ate}{ATE}{Average Treatment Effect}
\newacronym{llm}{LLM}{Large Language Model}
\newacronym{llms}{LLMs}{Large Language Models}
\newacronym{roc}{ROC curve}{Receiver Operating Characteristic curve}
\newacronym{auc}{AUC}{the area under the ROC}
\newacronym{curve}{DOC curve}{Distiller Operating Characteristic curve}
\newacronym{audc}{AUDC}{the area under the \acrshort{curve}}

\newacronym{tlpp}{TLPP}{Temporal Logical Point Process}

\newacronym{gnn}{GNN}{Graph Neural Network}

\newacronym{otd}{OTD}{Optimal Transport Distance}

\begin{document}

\author[A]{\fnms{Sishun}~\snm{Liu}\thanks{Corresponding Author. Email: sishun.liu@student.rmit.edu.au}}
\author[A]{\fnms{Ke}~\snm{Deng}}
\author[A]{\fnms{Xiuzhen}~\snm{Zhang}}
\author[B]{\fnms{Yan}~\snm{Wang}}

\address[A]{RMIT University, Australia}
\address[B]{Macquarie University, Australia}

\begin{frontmatter}

\paperid{218} 


\title{Learning \acrlong{mtpp} Explanations based on Counterfactual and Factual Reasoning}

\begin{abstract}

Neural network-based \acrfull{mtpp} models have been widely adopted to model event sequences in high-stakes applications, raising concerns about the trustworthiness of outputs from these models. This study focuses on \textit{\gls{ehd}}, aiming to identify the minimal and rational explanation, that is, the minimum subset of events in history, based on which the prediction accuracy of \acrshort{mtpp} matches that based on full history to a great extent and better than that based on the complement of the subset. This study finds that directly defining \gls{ehd} as \gls{ca} or \gls{fa} can result in irrational explanations. To address this issue, we define \gls{ehd} as a combination of \gls{ca} and \gls{fa}. This study proposes \acrfull{model} to solve \gls{ehd} with a series of deliberately designed techniques. Experiments demonstrate the correctness and superiority of \acrshort{model} over baselines regarding explanation quality and processing efficiency.



\end{abstract}
\end{frontmatter}

\mysection{Introduction}
\label{sec:intro}
The \textit{\acrfull{mtpp}}~\citep{daley_introduction_2003} is a well-defined stochastic process. The \acrshort{mtpp} models historical event sequences to a probability distribution that can be used to predict future events. Learning \acrshort{mtpp} by neural networks has been well investigated~\citep{shchurNeuralTemporalPoint2021}. 
These algorithms enable people to train and use \acrshort{mtpp} in high-stakes applications such as fake news mitigation~\citep{zhang_vigdet_2021,zhang_counterfactual_2022} and Electronic Health Record (EHR) modeling~\citep{enguehardNeuralTemporalPoint2020}. Decision making in high-stakes environments requires rational decisions~\citep{sahohRoleExplainableArtificial2023}, but how a neural network-based \acrshort{mtpp} maps historical events to a probability distribution remains unknown, raising concerns about the trustworthiness of outputs from these models in these high-stakes applications.

To fill the gap, this study targets \textit{\gls{ehd}}. We can identify a subset of events in history, based on which the prediction accuracy of \acrshort{mtpp} matches that based on the full history to a great extent. The identified subset can explain why \acrshort{mtpp} outputs in that particular way based on the full history. In practice, many subsets can satisfy the criterion. Among them, we are interested in the \textit{minimal and rational explanation}, \ie the minimum subset based on which the prediction accuracy of \acrshort{mtpp} is better than that based on the complement of the subset in history.
For example, an \acrshort{mtpp} model outputs the distribution of future events based on which the next earthquake is predicted. 
An explanation of the \acrshort{mtpp} model identifies which past events drive this prediction-for instance, a magnitude 5 earthquake at 11 pm five days ago and a magnitude 4 earthquake at 5 am thirty days ago. A minimal explanation requires that all identified earthquakes in history are indispensable to the model output, while a rational explanation further ensures that unselected earthquakes are irrelevant.




Recently, the explanation of model outputs has been studied for various black-box models. 
Some studies are based on \textit{\gls{ca}} to explain classifiers \citep{karlssonExplainableTimeSeries2018} and recommenders \citep{barkanCounterfactualFrameworkLearning2024}. Some are based on \textit{\Gls{fa}} for explaining classifiers~\citep{fernandezFactualCounterfactualExplanations2022} and \acrshort{gnn} models~\citep{linGenerativeCausalExplanations2021,liuMultiobjectiveExplanationsGNN2021,caiProbabilityNecessitySufficiency2025}. 
\Gls{ca} is a model-agnostic explanation method, which looks for the smallest modification to the model input that can change the output once applied and calls the obtained change a valid explanation~\citep{guidotti_counterfactual_2022}. However, we find that \gls{ca} alone does not guarantee a rational explanation for \acrshort{mtpp} models.
\Gls{fa} identifies the minimum set of features based on which the model returns the same result as that based on all features~\citep{xuCFE2CounterfactualEditing2024a}. Similar to \gls{ca}, we find that \gls{fa} alone cannot guarantee a rational explanation for \acrshort{mtpp} models either.

To avoid such irrational results, our study defines \gls{ehd} as a combination of counterfactual and \gls{fa}. Studies on \acrshort{gnn} identify subgraphs to explain \acrshort{gnn} outputs using \gls{ca} and \gls{fa}~\citep{tan_learning_2022, chen_grease_2022, xu_counterfactual_2022}.
Their methods cannot be applied to solve \gls{ehd}. They search for subgraphs in \acrshort{gnn} to explain the classification, while we target historical events to explain the output of \acrshort{mtpp}. In particular, unlike classifiers, 
the output of \acrshort{mtpp} models is a probability distribution, 
where it is not straightforward to define the output changes. Following~\citet{budhathokiWhyDidDistribution2021}, we decide that the distribution has changed if the discrepancy between the new and original distribution is greater than a threshold; otherwise, there is no change.

It is challenging to solve \gls{ehd}, which is a combinatorial problem. The optimal solution is intractable. We deliberately design a learning-based solution named \textit{\acrfull{model}} which optimizes parameters to select the fewest historical events by minimizing a surrogate hinge loss under counterfactual and factual constraints. In particular, the \acrshort{model} probes various combinations of historical events by enabling backpropagation with the Gumbel-softmax trick~\citep{bengio_estimating_2013,maddison_concrete_2016} and exploring the consistent relation between differentiable \(\norml{1}\) and nondifferentiable \(\norml{0}\) in the context of our problem. Our contributions are as follows:


\begin{enumerate}
    \item This work proposes the first study on \gls{ehd} to improve the trustworthiness of outputs from \acrshort{mtpp} models, an important problem, particularly in high-stakes applications.

    \item 
    This study finds the issue when defining \gls{ehd} as \gls{ca} or \gls{fa} and overcomes the issue by defining \gls{ehd} as the combination of \gls{ca} and \gls{fa}.
    \item This study develops \acrfull{model} with a series of deliberately designed techniques. The superiority of \acrshort{model} over baselines regarding explanation quality and processing efficiency has been verified by experiments. 
\end{enumerate}

\section{Related Works}\label{sec:related}


\subsection[Counterfactual and Factual Explanation]{Counterfactual and Factual Explanation}

\textbf{Counterfactual Explanation} has been used to analyze how 
classifiers make decisions on time-series data, images, texts, etc.
~\citep{karlssonExplainableTimeSeries2018,verma_counterfactual_2022,guidotti_counterfactual_2022}, how user behaviors and/or item features affect recommendations~\citep{liFairnessRecommendationFoundations2023, geSurveyTrustworthyRecommender2024}, and how \acrfull{gnn} models make predictions~\citep{prado-romeroSurveyGraphCounterfactual2024}.
For classifiers, \gls{ca} approaches are divided into two families, Optimization (OPT), an approach to resolve counterfactual explanations by minimizing specific losses using optimization algorithms~\citep{ramakrishnan_synthesizing_2020,parmentier_optimal_2021} and Heuristic Search Strategy (HSS), exploring various strategies like greedy~\citep{goyal2019counterfactual}, hill-climbing~\citep{lash2017generalized}, and best-first~\citep{10.25300/MISQ/2014/38.1.04} to search for counterfactual explanations.
For recommender systems, existing counterfactual explanation approaches include heuristic search~\citep{ghazimatin_prince_2020,tran_counterfactual_2021} and optimization~\citep{mu_alleviating_2022, barkanCounterfactualFrameworkLearning2024}. Among them, two concerns \gls{ca} related to the sequences of user activities~\citep{ghazimatin_prince_2020,tran_counterfactual_2021}. \citet{ghazimatin_prince_2020} proposed PRINCE to greedily remove minimal activities to replace current recommendations. In addition, \citet{tran_counterfactual_2021} discussed and addressed the limitation of PRINCE.
For \acrfull{gnn} models, \gls{ca} has been studied. 
\citet{abrate_counterfactual_2021} applied \gls{ca} to explain a \acrshort{gnn} model which describes human brains. 
\citet{zhang_page-link_2023} proposed PaGE-LINK to explain a learned \acrshort{gnn} recommender.

\textbf{Factual Explanation} has been used to help explain classifiers~\citep{fernandezFactualCounterfactualExplanations2022} and \acrshort{gnn} models~\citep{linGenerativeCausalExplanations2021,liuMultiobjectiveExplanationsGNN2021,caiProbabilityNecessitySufficiency2025}. For classifiers, 
\gls{fa} is believed insufficient in many situations and needs to be improved by \gls{ca}~\citep{fernandezFactualCounterfactualExplanations2022}. For the \acrshort{gnn} models, some researchers realized that the subgraph extracted for \gls{ca} or \gls{fa} can lead to incomplete explanations. To address this,~\citet{tan_learning_2022} proposed Counterfactual and Factual (CF\(^2\)) reasoning. CF\(^2\) merges \gls{ca} with factual reasoning to ensure that the obtained subgraph is complete. A similar idea has been explored by~\citet{chen_grease_2022} and~\citet{xu_counterfactual_2022}.

\textbf{Remarks}: No previous work investigates the explanation for \acrshort{mtpp} models based on \gls{ca} or \gls{fa}. Although irrelevant, we would like to mention that some \acrshort{mtpp} studies extract logic rules to strengthen confidence in future event prediction \citep{liExplainingPointProcesses2021,songLatentLogicTree2024,yangNeuroSymbolicTemporalPoint2024}, and others draw mark-wise attention maps for a better mark prediction of future events~\citep{zhang_self-attentive_2020,zhangLearningNeuralPoint2021}. In addition, some studies are based on well-defined mathematical formulas (\ie white box model) where the relationship between events at certain positions over time is parameterized, and the learned parameters can indicate the correlation strength between events to explain the occurrence of future events \citep{ideCardinalityRegularizedHawkesGrangerModel2021,wuLearningGrangerCausality2024}. They work on white-box models and, therefore, cannot help solve our problem for black-box \acrshort{mtpp} models.

\subsection{Counterfactual Prediction}
\label{sec:cp}
Please note that \textit{counterfactual prediction} and \textit{counterfactual explanation} are different problems. Counterfactual predictions refer to the process of estimating what would have happened to the output under a given alternative scenario~\citep{pearl_causality_2009}, while \gls{ca} searches for which scenario would lead to a different output~\citep{guidotti_counterfactual_2022}.
There are studies on counterfactual predictions for \acrshort{mtpp} models~\citep{gao2021causal, zhang_counterfactual_2022, noorbakhsh_counterfactual_2022, Hizli2023}.
\citet{zhang_counterfactual_2022} model the behavior of social media users using \acrshort{mtpp} and investigate how the probability of one user creating posts is influenced if engaged with misinformation.
\citet{noorbakhsh_counterfactual_2022} focus on unbiased sampling of an \acrshort{mtpp} model in an alternative scenario for next event predictions. \citet{gao2021causal} use counterfactual prediction to investigate the causal influence between event marks encoded in \acrshort{mtpp} models. 
~\citet{prosperi2020causal} and ~\citet{Hizli2023} apply counterfactual prediction on healthcare \acrshort{mtpp} models to estimate direct and indirect effects of a treatment.
Furthermore, counterfactual predictions are valuable in applications such as policy evaluation~\citep{ferraro2009counterfactual} and decision-making~\citep{schulam2017reliable}, where understanding the impact of hypothetical changes can guide more informed and effective strategies. In short, the counterfactual prediction and explanation are different problems, and the methods for counterfactual prediction on \acrshort{mtpp} cannot be applied to solve \gls{ehd}.

\section{Problem Definition}
\mysubsection[MTPP]{\acrlong{mtpp}}
\label{sec:mtpp}

The \acrfull{mtpp} describes a random process of an event sequence \(\data{x} = (x_1, x_2, \cdots, x_n)\).
Each event \(x_i = (m_i, t_i)\) comprises a categorical mark \(m_i \in \set{M} = \{k_1, k_2, \cdots, k_M\}\) and its occurrence time \(t_i\).
This paper considers the simple \acrshort{mtpp}, which only allows at most one event at every time, thus \(t_i<t_j\) if \(i<j\).
Given the history up to (exclusive) the current time \(t\), denoted as \(\history\), the \gls{cif} \(\lambda^*(m, t)\) is the probability that an event with mark \(m\) will happen at time \(t\)~\citep{daley_introduction_2003}:\footnote{The asterisk reminds that this function conditions on history.} 
\begin{equation}
\label{eqn:def_mtpp_intensity}
    \lambda^*\left(m,t\right) = \lim_{\Delta t \rightarrow 0}{\dfrac{P\left(m ,t \in (t, t+\Delta t]\middle|\history\right)}{\Delta t}}.
\end{equation}
With \(\lambda^*(m, t)\), we can define the joint probability distribution \(p^*(m, t)\) of the next event whose mark is \(m\) and the time is \(t\).
\begin{equation}
    \label{eqn:mtpp_p}
    p^*\left(m, t\right) = \lambda^*\left(m, t\right)\exp{-\sum_{k \in \set{M}}{\int_{t_l}^{t}{\lambda^*(k, \tau)\mathrm{d}\tau}}}.
\end{equation}
The \acrfull{nll} loss on \(\data{x}\) observed in a time interval \([t_0,T]\) is:
\begin{equation}
\label{eqn:nll_of_mtpp}
    L = -\log p(\data{x}) = - \sum_{i = 1}^{n}{\log \lambda^*(m_i, t_i)} + \sum_{k \in \set{M}}{\int_{t_0}^{T}{\lambda^*(k, \tau)\mathrm{d}\tau}}.
\end{equation}
\cref{eqn:nll_of_mtpp} is the training loss of \acrshort{mtpp} models. Most recent \acrshort{mtpp} models are based on neural networks. When using \(p^*(m, t)\) to predict the next event, we first predict the time of the next event \(\bar{t}\) as the expectation of \(p^*(t) = \sum_{m \in \set{M}}{p^*(m, t)}\), then predict the mark of the next event \(\bar{m}\) as the most probable mark at time \(\bar{t}\), \ie \(\bar{m} = \arg\max_{m \in \set{M}}{p^*(m, \bar{t})}\)~\citep{shchurNeuralTemporalPoint2021,panosDecomposableTransformerPoint2024}.

\mysubsection[Counterfactual and Factual Explanation]{Counterfactual and Factual Explanation}
\label{sec:ce}

Counterfactual explanation is a \textit{perturbation-based technique} belonging to the \textit{feature attribution} family. Feature attribution aims at measuring the relevance between input features and model outputs. Perturbation-based technique measures relevance by making small and controlled changes to the input data to gain insights into how the model made that decision~\citep{zhaoExplainabilityLargeLanguage2024}. On the other hand, \gls{fa} is based on \textit{abductive reasoning}, in which a conclusion is drawn based on the best explanation for a given set of observations~\citep{huangReasoningLargeLanguage2023}. 

Suppose the input of a model \(\model\) is \(\data{x}\), and the corresponding model output is \(\data{y} = \model(\data{x})\). \Gls{ca} aims at the smallest change in feature values that can translate to a different output of a model. Specifically, \gls{ca} searches for \(\data{x}^{\prime}\) such that \(\model(\data{x}^{\prime}) \neq y\), and the difference between \(\data{x}\) and \(\data{x}^{\prime}\) is minimal\cite{guidotti_counterfactual_2022}: 

\begin{equation}
    \begin{aligned}
    \min_{\data{x}^{\prime}}&\ d(\data{x}, \data{x}^{\prime}) \\
    \text{s.t.}\ \ \ &\model(\data{x}^{\prime}) \neq \model(\data{x}) \\
    \end{aligned}\label{eqn:ce_optim}
\end{equation}
where \(d(\data{x}, \data{x}^{\prime})\) measures the difference between \(\data{x}\) and \(\data{x}^{\prime}\). 

\Gls{fa} aims at the smallest subset in feature values that can lead to the observed output. Specifically, \gls{fa} searches for \(\data{x}^{\prime}\) such that \(\model(\data{x}^{\prime}) = y\), and the size of \(\data{x}^{\prime}\) is minimal~\cite{tan_learning_2022}: 
\begin{equation}
    \begin{aligned}
    \min_{\data{x}^{\prime}}&\ |\data{x}^{\prime}| \\
    \text{s.t.}\ \ \ &\model(\data{x}^{\prime}) = \model(\data{x}) \\
    \end{aligned}\label{eqn:fa_optim}
\end{equation}
where \(|\data{x}^{\prime}|\) the size of \(\data{x}^{\prime}\). Both \gls{ca} and \gls{fa} adhere to Occam's razor that one should prefer the hypothesis requiring the fewest assumptions about the same prediction\citep{gauch2003scientific}.


\mysubsection{Problem Statement and Formulation}
\label{sec:statement}
We extract a subsequence \((h_1, \cdots, h_{j}, x_{1}, \cdots, x_n)\) from an event sequence. The first part \((h_1, \cdots, h_{j})\), denoted as \(\history\), is the history relative to the second part \((x_{1}, \cdots, x_n)\), denoted as \(\data{x}\). Given \(\history\) and \(\data{x}_{<i}\coloneq (x_1,\cdots,x_{i-1})\subset \data{x}\), a black-box \acrshort{mtpp} model \(\mathcal{M}\) can be used to estimate the distribution of the next event where the probability of \(x_i\) is denoted as \(p(x_i|x_{<i}, \history)\). The higher \(p(x_i|x_{<i}, \history)\) means the estimated distribution fits \(x_i\) better and implies a more accurate prediction of the next event. To evaluate \(p(x_i|x_{<i}, \history)\) for all \(x_{i}\in \data{x}\), a suitable measure is \textit{perplexity}.
The perplexity is defined as:
\begin{equation}
    \label{eqn:posterior}
    \mathrm{ppl}(p(\data{x}|\history))
    = \exp{-\frac{1}{|\data{x}|} \log \prod_{i = 1}^{n}{p(x_i|\data{x}_{<i}, \history)}}.
\end{equation}
\(\mathrm{ppl}(\cdot)\) is perplexity. 
A lower perplexity means higher \(p(x_i|x_{<i}, \history)\) for all \(x_{i}\in \data{x}\) based on \(\history\), indicating higher prediction accuracy. 

\begin{figure*}[ht]
    \centering
    \includegraphics[width=0.8\textwidth]{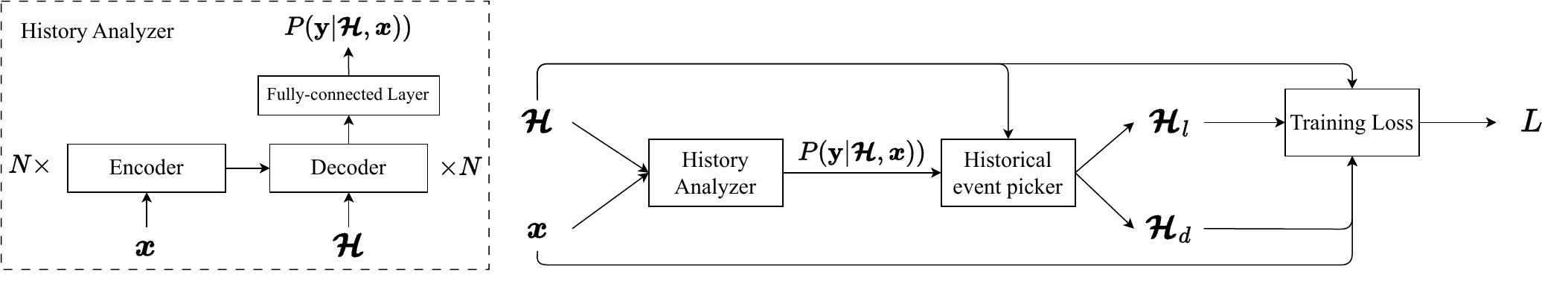}
    \caption{Architecture of \acrfull{model}.} 
    \label{fig:ehd_model}
\end{figure*}

This study aims to identify the rational explanation from \(\history\), denoted \(\history_d\). First, \(\history_d\) is an explanation, so $\mathrm{ppl}(p(\data{x}|\history)) \geqslant \epsilon_d* \mathrm{ppl}(p(\data{x}|\history_{d}))$, where $\epsilon_d\in (0,1)$ and \(\epsilon_d\) should be high. This ensures that the prediction accuracy of \acrshort{mtpp} based on $\history_{d}$ matches that based on $\history$ to a great extent.
Second, \(\history_d\) is rational and minimal. Being rational means that \(\mathrm{ppl}(p(\data{x}|\history_d)) < \mathrm{ppl}(p(\data{x}|\history_l))\) where \(\history_l = \history\setminus\history_d\). This ensures that the prediction accuracy of \acrshort{mtpp} based on \(\history_d\) is better than that based on \(\history_l\). Being minimal means that no other rational explanations have fewer events than \(\history_d\).



Our study finds that defining \gls{ehd} as \gls{ca} or \gls{fa} can lead to irrational explanations, \ie \(\mathrm{ppl}(p(\data{x}|\history_d)) > \mathrm{ppl}(p(\data{x}|\history_l))\), as shown in \cref{fig:compare_conv_ca}.
To address this issue, we define \gls{ehd} as a combination of \gls{ca} and \gls{fa}:
\begin{equation}
\begin{aligned}
\min_{\history_{d}\subseteq\history}&|\history_d| \\
\text{s.t.}\ \ \ &\log \mathrm{ppl}(p(\data{x}|\history)) - \log\mathrm{ppl}(p(\data{x}|\history_{l})) \leqslant \log \epsilon_l, \\
\ \ \ &\log \mathrm{ppl}(p(\data{x}|\history)) - \log\mathrm{ppl}(p(\data{x}|\history_{d})) \geqslant \log\epsilon_d, \\
\ \ \ \ \ &\epsilon_d>\epsilon_l \\
\end{aligned}\label{eqn:new_optimization_problem}
\end{equation}
where
\(\history_d\cap \history_l=\varnothing\), \(\history_d\cup \history_l=\history\), \(\epsilon_{l}\) and \(\epsilon_{d}\) are hyperparameters. The first constraint is counterfactual reasoning to enforce the prediction accuracy of \acrshort{mtpp} based on $\history_l$ low by applying a threshold \(\epsilon_{l}\in (0,1)\). The second constraint is the factual reasoning to ensure that the prediction accuracy of \acrshort{mtpp} based on $\history_l$ is high by applying threshold \(\epsilon_{d}\in (0,1)\). The third constraint guarantees that the prediction accuracy of \acrshort{mtpp} based on $\history_d$ is higher than $\history_l$. 

Users can control the resulting explanation by setting \(\epsilon_d\) and \(\epsilon_l\). Because the events in \(\history_{d}\) serve as a rational explanation why \acrshort{mtpp} outputs the particular probability distribution based on the full history, \(\epsilon_d\) should be high enough and \(\epsilon_l\) should be low enough. However, the extremely low \(\epsilon_l\) can lead to \(\history_d\) with all events in history. To avoid this situation, \(\epsilon_d=0.9\) and \(\epsilon_l=0.5\) or $0.6$ are default settings in this study. With the definition of \gls{ehd}, we have the following proposition:  
\begin{proposition}
    \Gls{ehd} defined in \cref{eqn:new_optimization_problem} always has a solution for any \(\epsilon_l \in (0, 1)\), \(\epsilon_d \in (0, 1)\), and \(\epsilon_d>\epsilon_l\).
    \label{prop:ehd_solvable}
\end{proposition}
\begin{proof}
    By connecting the two constraints in \cref{eqn:new_optimization_problem}, we have:
    \begin{equation}
    \frac{\mathrm{ppl}(p(\data{x} | \history_{l}))}{\mathrm{ppl}(p(\data{x} | \history_{d}))} \geqslant \frac{\epsilon_d}{\epsilon_l}
    \label{eqn:new_const}
    \end{equation}
    For any \(\epsilon_l \in (0, 1)\) and \(\epsilon_d \in (0, 1)\) where \(\epsilon_d>\epsilon_l\), we can always move more events from \(\history_l\) to \(\history_d\) so that \cref{eqn:new_const} is satisfied. In the extreme case that \(\frac{\epsilon_d}{\epsilon_l}\) is an any large number, all events in \(\history_l\) can be moved to \(\history_d\) so that \(\history_l=\varnothing\); then we get \(\mathrm{ppl}(p(\data{x} | \history_{l}))\rightarrow +\infty\) that guarantees the inequation in \cref{eqn:new_const} held.

\end{proof}

\mysection[MTPP-CHD]{\acrlong{model}}\label{sec:model}

For \gls{ehd}, we propose \acrfull{model}, which is sketched in \cref{fig:ehd_model}. \acrshort{model} consists of three components. The first component, \textit{history analyzer}, processes \(\history\) and \(\data{x}\) using an encoder-decoder transformer and a fully connected layer. The output is \(p(\vec{y}|\history, \data{x})\), a distribution that tells which events in \(\history\) are likely to be in \(\history_d\) or \(\history_l\). All trainable parameters are in the first component. The second component, \textit{historical event picker}, 
generate \(\history_d\) and \(\history_l\) based on \(p(\vec{y}|\history, \data{x})\). The third component, \textit{training loss}, employs the black-box \acrshort{mtpp} model \(\mathcal{M}\) to evaluate the \(\history_d\)s from the second component, and the loss is used for training \acrshort{model}. The third component only exists during training.

\mysubsection[MTPP-CHD training]{Training of \acrshort{model}}  

The training dataset contains \((\history\), \(\data{x})\) pairs extracted from the event sequences on which \(\mathcal{M}\) can be used to predict the time and mark of the next events. Training \acrshort{model} begins by initializing the parameters of the \textit{history analyzer}. Then, the history analyzer processes each \((\history\), \(\data{x})\) pair in the training dataset 
using an encoder-decoder transformer to represent each event in \(\history\) so that it is aware of other events in \(\history\) and the events in \(\data{x}\). Next, the representations of events in \(\history\) are fed to a fully connected layer and the output is \(p(\vec{y}|\history, \data{x}) = \prod_{i = 1}^{j}{p(y_i|\history, \data{x})}\) where \(j=|\history|\). For each element \(y_i\in \vec{y}\), \(p(y_i|\history, \data{x})\) is a categorical distribution of two categories, \ie \(y_i \in \{0,1\}\). 
If \(p(y_i = 1|\history, \data{x})\) is greater, the \(i\)-th event in \(\history_l\) is more likely moved go to \(\history_d\), otherwise more likely remains in \(\history_l\).


Next, the \textit{historical event picker} draws samples \(\hat{\vec{y}}\) from \(p(\vec{y}|\history, \data{x})\), each sample leading to unique \(\history_d\) and \(\history_l\). Specifically, \(\hat{\vec{y}} = (\hat{y_1}, \hat{y_2}, \cdots, \hat{y_j})\) where \(\hat{y}_i\) (\(1\leq i\leq j\)) is drawn from the categorical distribution \(p(y_i|\history, \data{x})\).  Because we need gradients from \(\history_l\) and \(\history_d\) to train \(p(\vec{y}|\history, \data{x})\) through \(\hat{\vec{y}}\), the sampling process must be differentiable. Therefore, we use the Gumbel-softmax trick~\citep{maddison_concrete_2016} to differentiably get \(\hat{y_i}\), which is:
\begin{equation}
    \hat{y_i} = \frac{\exp{(\log p(y_i = 1|\history, \data{x}) + g_1)/\tau}}{\sum_{C\in\{0, 1\}}{\exp{(\log p(y_i = C|\history, \data{x}) + g_C)/\tau}}}
\end{equation}
where \(g_C\) is a sample from the standard Gumbel distribution, and \(\tau\) is the temperature. In our case, we set \(\tau = 0\) to ensure \(\hat{y}_i\) is either 0 or 1 and employ the Straight Through trick~\citep{bengio_estimating_2013}, so that the sampling process is still differentiable. For each \(\hat{\vec{y}}\), we initialize \(\history_l\) as a copy of \(\history\) and \(\history_d=\varnothing\). If \(\hat{y}_i=1\), \(i\)-th event in \(\history_l\) is moved to \(\history_d\), otherwise, remains in \(\history_l\). This way, we have many \(\hat{\vec{y}}\)s drawn from \(p(\vec{y}|\history, \data{x})\) and each leads to unique \(\history_d\) and \(\history_l\). It is more likely the generated \(\history_d\) consists of the \(i\)-th event if \(p(y_i=1|\history, \data{x})\) is higher than \(p(y_i=0|\history, \data{x})\). If \(\history_d\) is desired, it indicates the probability is ideal; otherwise, needs to be optimized. A similar method has been used for rationalization in natural language processing to search for an optimal combination of sentences related to a claim ~\citep{lei_rationalizing_2016}.






The third component evaluates \(\history_d\)s and their corresponding \(\history_l\)s from the second component for the loss. According to \cref{eqn:new_optimization_problem}, the loss comprises two aspects: \(\losse\) for enforcing perplexity-based constraints and \(\lossn\) for minimizing the length of \(\history_d\). The training loss \(L\) of \acrshort{model} is the weighted sum of \(\lossn\) and \(\losse\):
\begin{equation}
\label{eqn:loss}
    L = \alpha \lossn + \beta \losse.
\end{equation}
Specifically, the loss \(\losse\) is: 
\begin{equation}
    \label{eqn:L_e}
    \losse 
    = \expect_{\hat{\vec{y}}\sim p(\vec{y}|\history, \data{x})}(L_l(\hat{\vec{y}}) + L_d(\hat{\vec{y}}))
\end{equation}
where \(L_{l}(\hat{\vec{y}})\) and \(L_{d}(\hat{\vec{y}})\) are loss for the first and second constraints in \cref{eqn:new_optimization_problem}, respectively. Inspired by~\citep{mothilal_explaining_2020,tan_counterfactual_2021}, we enforce perplexity-based constraints in a differentiable way by using two surrogate hinge losses:
\begin{align}
    \begin{split}
    \label{eqn:hinge_loss}
    L_{l}(\hat{\vec{y}}) &= \max(\log \frac{\mathrm{ppl}(p(\data{x} | \history))}{\mathrm{ppl}(p(\data{x} | \history_l))} - \log\epsilon_l, 0).
    \end{split} \\
    \begin{split}
    L_{d}(\hat{\vec{y}}) &= \max(\log \frac{\mathrm{ppl}(p(\data{x} | \history_d))}{\mathrm{ppl}(p(\data{x} | \history))} + \log\epsilon_d, 0).
    \end{split}
\end{align}
where the conditional probability distribution \(p(\data{x} | \history)\), \(p(\data{x} | \history_d)\), and \(p(\data{x} | \history_l)\) are estimated using \(\mathcal{M}\).

The loss \(\lossn\) aims to minimize the number of events in \(\history_d\). For each element \(\hat{y}_i\in \hat{\vec{y}}\), if \(\hat{y}_i=1\), the \(i\)-th event is moved from \(\history_l\) to \(\history_d\); otherwise, remains in \(\history_l\). 
So, minimizing the number of events in \(\history_d\) equals to minimizing \(\norml{0}\) of \(\hat{\vec{y}}\). However, the \(\norml{0}\) is not differentiable. As a workaround, some studies optimize the differentiable \(\norml{1}\)~\citep{tan_counterfactual_2021}. Generally, optimizing \(\norml{1}\) of a vector \(\vec{a}\in \set{R}^{d}\) has limited effects on optimizing \(\norml{0}\) because there is no consistent relation between them. \(\norml{0}\) can decrease, stay unchanged, or even increase when \(\norml{1}\) decreases. Fortunately, \(\norml{1}\) of \(\hat{\vec{y}}\) has a consistent relation with \(\norml{0}\) of \(\hat{\vec{y}}\) because the elements in \(\hat{\vec{y}}\) are either \num{0} or \num{1}. This means that \(\norml{0}\) is always equal to \(\norml{1}\) and therefore minimizing \(\hat{\vec{y}}\)'s \(\norml{1}\) is equivalent to minimizing \(\hat{\vec{y}}\)'s \(\norml{0}\). We define \(\lossn\) as the normalized \(\norml{1}\):
\begin{equation}
\label{eqn:l_n}
    \lossn = \frac{\norm{\hat{\vec{y}}}{1}}{|\hat{\vec{y}}|}.
\end{equation}

\mysubsection[CHD]{Inference of \acrshort{model}}



Given history \(\history\) and \(\data{x}\), the \textit{history analyzer} outputs \(p(\vec{y}|\history, \data{x})\) for inference. Then, the \textit{historical event picker} returns \(\history_d\) based on \(p(\vec{y}|\history, \data{x})\). Specifically, the elements \(y_i\in\vec{y}\) are sorted in descending order of \(p(y_i=1|\history, \data{x})\). 
Initially, \(\history_l\) is a copy of \(\history\), \(\history_d=\varnothing\), and \(k=1\). The event corresponding to the top-k element(s) is moved to \(\history_d\) from \(\history_l\). Using \(\mathcal{M}\), \(\mathrm{ppl}(p(\data{x}|\history_l))\) and \(\mathrm{ppl}(p(\data{x}|\history_d))\) are calculated. If the constraints in \cref{eqn:new_optimization_problem} are satisfied, \(\history_d\) is returned. If not, \(k=k+1\) and the same process is taken until the constraints in \cref{eqn:new_optimization_problem} are satisfied and \(\history_d\) is returned.

\mysection{Experiments}

This section
(i) compares our definition of \gls{ehd} with those defined as \gls{ca} and \gls{fa}, respectively, (ii) evaluates the stability of \acrshort{model} under different \((\epsilon_l\), \(\epsilon_d)\) pairs, (iii) checks the explanatory capability of \acrshort{model}-generated \(\history_d\), and (iv) compares \acrshort{model} with baselines for solving \gls{ehd}. We run each experiment five times with different random seeds, and their mean and standard deviation (1-sigma) are reported.

\acrshort{model} can work with any existing black-box \acrshort{mtpp} model \(\mathcal{M}\) 
that provides \(p^*(m, t)\). As a well-studied research problem, the existing state-of-the-art \acrshort{mtpp} models demonstrate comparable performance for predicting the next events~\citep{shchurNeuralTemporalPoint2021}. Without loss of generality, our experiments adopt FullyNN~\citep{omi_fully_2019} because it is a widely accepted \acrshort{mtpp} model for its simplicity, stability, and good performance. More details about FullyNN are in \cref{app:mtpp_model}.

The default hyperparameters of \acrshort{model} including the setting of \(\epsilon_l\) and \(\epsilon_d\) 
are in \cref{app:datasets}. 
We train and evaluate \acrshort{model} and baselines on Xeon Gold 6132 CPUs with 256GB RAM with a V100 GPU.

\partopic{Baseline Models}
\textit{To our knowledge, no previous studies have investigated \gls{ehd}. So, there are no baselines from existing studies to compare with.} Two following baselines are compared: 
\begin{itemize}
    \item \textbf{\acrfull{gs}} is a widely accepted baseline in \gls{ca} research (\eg \citep{ghazimatin_prince_2020,tran_counterfactual_2021}) to evaluate performance. \acrshort{gs} solves \gls{ehd} by incrementally moving from \(\history_l\) (a copy of \(\history\) initially) to \(\history_d\) the event that increases the gap between \(\log\mathrm{ppl}(\data{x}|\history_l)\) and \(\log\mathrm{ppl}(\data{x}|\history_d)\) the most and this process repeats until the two constraints in \cref{eqn:new_optimization_problem} are satisfied.
    \item \textbf{\acrfull{rd}} is another widely accepted baseline in \gls{ca} research (\eg \citep{tan_counterfactual_2021}). \acrshort{rd} randomly moves \(Q\) events from \(\history_l\) (a copy of \(\history\) initially) to \(\history_{d}\) and calculates the gap (i) between \(\log\mathrm{ppl}(\data{x}|\history)\) and \(\log\mathrm{ppl}(\data{x}|\history_l)\) and (ii) between \(\log\mathrm{ppl}(\data{x}|\history)\) and \(\log\mathrm{ppl}(\data{x}|\history_d)\). This is repeated multiple times, and the average for each gap is recorded. \(Q\) starts from 0. \acrshort{rd} stops and returns \(Q\) when the average gap satisfies the two constraints in \cref{eqn:new_optimization_problem}; otherwise, increase \(Q\) by 1 and repeat the previous process. \acrshort{rd} serves as a simple solution showing the performance lower bound. 
\end{itemize}
Existing studies on counterfactual prediction for \acrshort{mtpp} are not baselines because counterfactual prediction and \gls{ehd} are fundamentally different, as stated in \cref{sec:cp}. The brute force is infeasible because solving a combinatorial problem like \gls{ehd} is NP-hard ~\citep{karp_reducibility_1972}.

\partopic{Evaluation Metrics} We quantitatively evaluate the explanatory capability of \(\history_d\) from \acrshort{model} by the distance between real event sequence \(\data{x}\) and the predicted event sequence \(\data{x}^{\prime}\) using the \acrshort{mtpp} model \(\mathcal{M}\) based on \(\history_d\), where \(|\data{x}^{\prime}| = |\data{x}|\). A smaller distance indicates a better \(\history_d\). A widely accepted metric for the distance between two event sequences is the \textit{\acrfull{otd}} \citep{meiImputingMissingEvents2019,panosDecomposableTransformerPoint2024}, which is the minimal cost of aligning \(\data{x}^{\prime}\) to \(\data{x}\) by moving, inserting, and deleting events. We qualitatively evaluate the explanatory capability of \(\history_d\) by checking if the generated explanation is consistent with common sense.
To compare \acrshort{model} with baselines, we consider two metrics following the evaluation protocol in \citep{ghazimatin_prince_2020}: (i) the length of \(\history_d\), while two constraints in \cref{eqn:new_optimization_problem} are satisfied, and (ii) how much time the solver spends to obtain \(\history_d\). For the former, we calculate the average length of \(\history_d\), \(\bar{|\history_d|}\), returned by \acrshort{model} and baselines, respectively.
\begin{equation}
    \bar{|\history_d|} = \frac{1}{|\data{T}|}\sum_{(\history, \data{x})\in \data{T}}{|\history_d|}.
\end{equation}
where \(\data{T}\) is the test dataset. 
Lower \(\bar{|\history_d|}\) indicates better performance. 
For the latter, we record how much time \acrshort{model} and baselines, respectively, spend to obtain \(\history_d\) in all $(\history,\data{x})$ pairs in \(\data{T}\). Lower time means higher processing efficiency.

\partopic{Datasets}\label{sec:datasets}

We test \acrshort{model} and baselines on three popular real-world datasets: \gls{retweet}~\citep{zhao_seismic_2015}, \gls{so}\footnote{\url{https://stackoverflow.com}}~\citep{Leskovec2014SNAPD} and \gls{yelp}\footnote{\url{https://www.yelp.com}}. \cref{tab:dataset_features} reports the statistical information of these datasets. All subsequences with \(n=|\history|+|\data{x}|\) events are extracted from these datasets. Further, each subsequence is split into \(\history\) and \(\data{x}\). Each dataset has 5 different \(|\history|\) and \(|\data{x}|\) settings, which are presented in \cref{tab:hyperparameter_dataset}. Detailed descriptions about these three datasets are available in \cref{app:datasets}.


\begin{table}[!ht]
    \small
    \caption{The statistical information of datasets where the number of sequences, events, and marks are in the first three columns, \(\bar{\tau}\) and \(\sigma(\tau)\) are the mean and standard deviation of the time intervals between adjacent events, \(t_0\) and \(T\) are the earliest start time and the latest end time of all sequences.}
    \centering
    \begin{tabular}{lccc}
        \toprule
                      & {Retweet}   & {StackOverflow} & {Yelp}     \\
        \midrule
        Sequences     & 24000       & 6633          & 4022     \\
        Events        & 2610102     & 480414        & 409946   \\
        Marks         & 3           & 22            & 3        \\
        \(\bar{\tau}\) & 2574        & 0.8747        & 7.2644   \\
        \(\sigma(\tau)\) & 16302       & 1.2091        & 13.410   \\
        \(t_0\)         & 0           & 1324          & 0        \\
        \(T\)           & 604799      & 1390          & 751      \\
        \bottomrule
    \end{tabular}
    \label{tab:dataset_features}
\end{table}
\begin{table}[!ht]
    \small
    \centering
    \caption{Settings of \(|\history|\) and \(|\data{x}|\) in experiments for each dataset.}
    \begin{tabular}{lc}
       \toprule
                     & (\# of events in \(\data{x}\), \# of events in \(\history\)) \\
       \midrule
       Retweet       &          (10, 25), (10, 30), (10, 35), (15, 35), (20, 35)               \\
       StackOverflow &          (15, 40), (15, 45), (15, 50), (20, 50), (25, 50)               \\
       Yelp          &          (10, 25), (10, 30), (10, 35), (15, 35), (20, 35)               \\
       \bottomrule
    \end{tabular}
    \label{tab:hyperparameter_dataset}
    \vskip -0.1in
\end{table}







\begin{figure*}[ht]
    \captionsetup[subfigure]{justification=centering}
    \centering
    \begin{subfigure}{0.2\textwidth}
        \includegraphics[width=\textwidth]{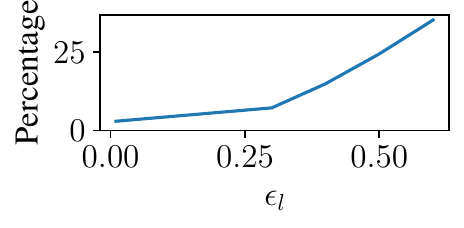}
    \end{subfigure}
    \begin{subfigure}{0.2\textwidth}
        \includegraphics[width=\textwidth]{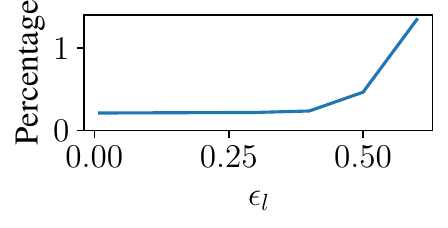}
    \end{subfigure}
    \begin{subfigure}{0.2\textwidth}
        \includegraphics[width=\textwidth]{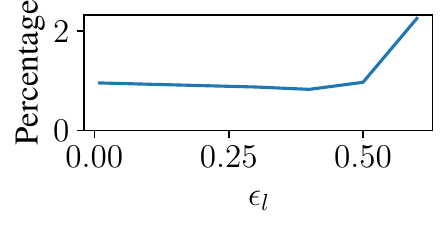}
    \end{subfigure} \\ 
    \begin{subfigure}{0.2\textwidth}
        \includegraphics[width=\textwidth]{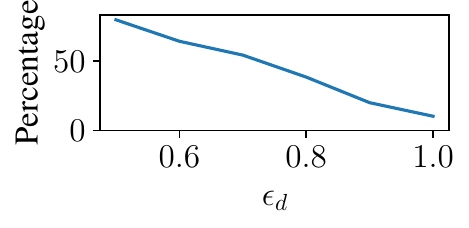}
        \caption{\gls{so} \\ (\( |\data{x}| = 15\), \(|\history| = 40\))}
    \end{subfigure}
    \begin{subfigure}{0.2\textwidth}
        \includegraphics[width=\textwidth]{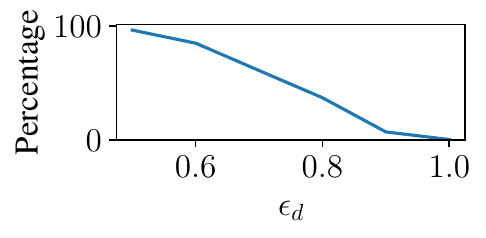}
        \caption{\gls{retweet} \\ (\( |\data{x}| = 10\), \(|\history| = 25 \))}
    \end{subfigure}
    \begin{subfigure}{0.2\textwidth}
        \includegraphics[width=\textwidth]{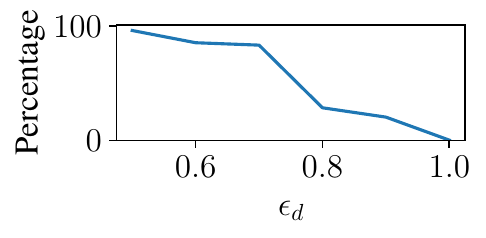}
        \caption{\gls{yelp} \\ (\(|\data{x}| = 10\), \(|\history| = 25\))}
    \end{subfigure}
    \vspace{0.5cm}
    \caption{First row: the percentage of irrational explanations in terms of threshold \(\epsilon_l\) when \gls{ehd} is defined as \gls{ca}. Second row: the percentage of irrational explanations in terms of threshold \(\epsilon_l\) when \gls{ehd} is defined as \gls{fa}.}
    \label{fig:compare_conv_ca}
\end{figure*}
\begin{figure*}[ht]
    \begin{subfigure}{\textwidth}
        \centering
        \begin{subfigure}{0.14\textwidth}
            \includegraphics[width=\textwidth]{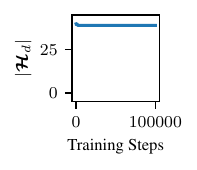}
        \end{subfigure}
        \begin{subfigure}{0.14\textwidth}
            \includegraphics[width=\textwidth]{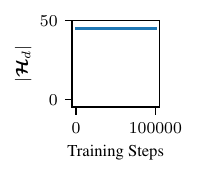}
        \end{subfigure}
        \begin{subfigure}{0.14\textwidth}
            \includegraphics[width=\textwidth]{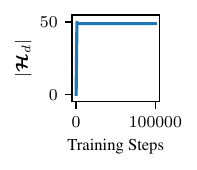}
        \end{subfigure}
        \begin{subfigure}{0.14\textwidth}
            \includegraphics[width=\textwidth]{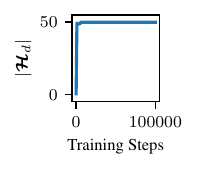}
        \end{subfigure}
        \begin{subfigure}{0.14\textwidth}
            \includegraphics[width=\textwidth]{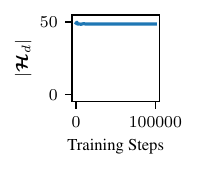}
        \end{subfigure}
        \caption{The number of event in \(\mathcal{H}_{d}\) returned by \acrshort{model} trained by minimizing \(\losse\) only.} 
        \vspace{0.5cm}
        \label{fig:l_te_so}
    \end{subfigure}
    \begin{subfigure}{\textwidth}
        \centering
        \begin{subfigure}{0.14\textwidth}
            \includegraphics[width=\textwidth]{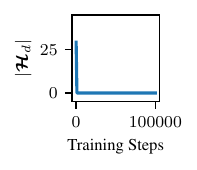}
        \end{subfigure}
        \begin{subfigure}{0.14\textwidth}
            \includegraphics[width=\textwidth]{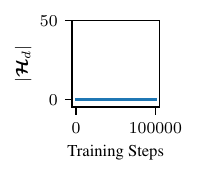}
        \end{subfigure}
        \begin{subfigure}{0.14\textwidth}
            \includegraphics[width=\textwidth]{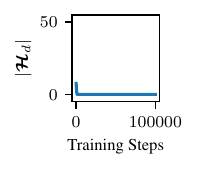}
        \end{subfigure}
        \begin{subfigure}{0.14\textwidth}
            \includegraphics[width=\textwidth]{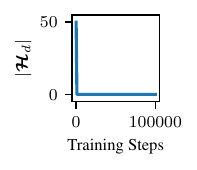}
        \end{subfigure}
        \begin{subfigure}{0.14\textwidth}
            \includegraphics[width=\textwidth]{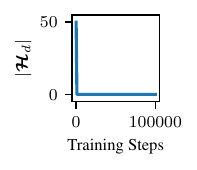}
        \end{subfigure}
        \caption{The number of event in \(\mathcal{H}_{d}\) returned by \acrshort{model} trained by minimizing \(\lossn\) only.} 
        \vspace{0.5cm}
        \label{fig:l_n_so}
    \end{subfigure}
    \caption{Effectiveness of \(\losse\) and \(\lossn\) (from left to right: \((|\data{x}_o|, |\history_f|)\) \(=\) \((15, 40),\) \((15, 45),\) \((15, 50),\) \((20, 50),\) \((25, 50)\)).} 
    \vspace{0.5cm}
    \label{fig:efficiency_so}
\end{figure*}

\mysubsection[EHD vs. Counterfactual Explanation]{Benefit of combining \gls{ca} and \gls{fa}}\label{sec:new_constraints}

If we define \gls{ehd} as \gls{ca} or \gls{fa}, our study shows that it can return irrational explanations, \ie the prediction accuracy of \acrshort{mtpp} based on \(\history_d\) is lower than that based on \(\history_l\), regardless of thresholds. In \cref{fig:compare_conv_ca}, the first row reports the percentage of irrational explanations for different threshold settings \(\epsilon_l\), where \gls{ehd} is defined as \gls{ca}.
The results of the experiment reveal consistent irrational explanations no matter what the value of \(\epsilon_l\) is, particularly for \gls{so}. The second row in \cref{fig:compare_conv_ca} reports the percentage of irrational explanations for different threshold settings \(\epsilon_d\), where \gls{ehd} is defined as \gls{fa}. The result shows that the irrational explanations persist. To avoid such irrational results, it is indispensable to define \gls{ehd} as a combination of \gls{ca} and \gls{fa} and only return \(\history_d\) that meets the constraints of \cref{eqn:new_optimization_problem}. In this way, \(\history_d\) is always rational.

\begin{figure*}[ht]
    \captionsetup[subfigure]{justification=centering}
    \centering
        \begin{subfigure}{0.17\textwidth}
            \includegraphics[width=\textwidth]{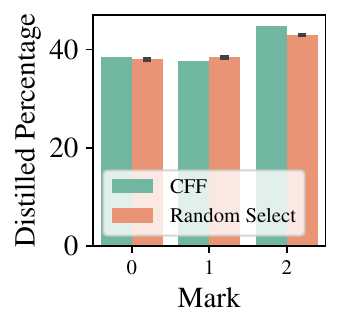}
        \end{subfigure}
        \begin{subfigure}{0.17\textwidth}
            \includegraphics[width=\textwidth]{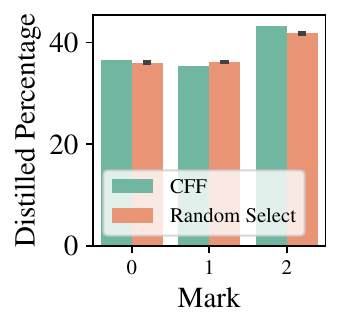}
        \end{subfigure}
        \begin{subfigure}{0.17\textwidth}
            \includegraphics[width=\textwidth]{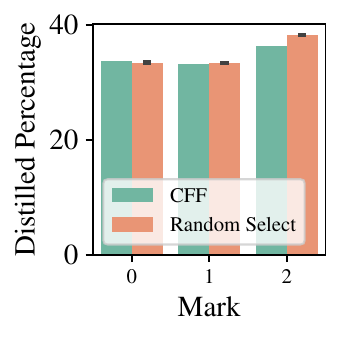}
        \end{subfigure}
        \begin{subfigure}{0.17\textwidth}
            \includegraphics[width=\textwidth]{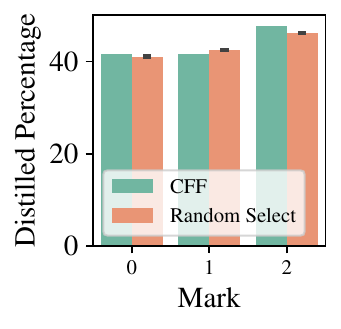}
        \end{subfigure}
        \begin{subfigure}{0.17\textwidth}
            \includegraphics[width=\textwidth]{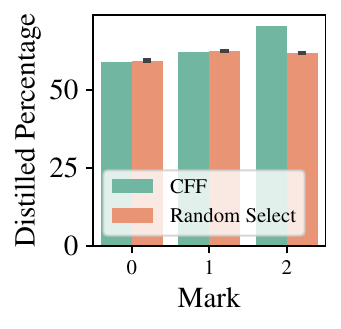}
        \end{subfigure}
    \caption{The percentage of events for different marks in \(\history_d\) returned by \acrshort*{model} and \acrfull*{rd} on test date of Retweet (from left to right: \((|\data{x}|, |\history|)\) \(=\) \((10, 25),\) \((10, 30),\) \((10, 35),\) \((15, 35),\) \((20, 35)\)). All results pass the significance test with p-value 0.}
    \vspace{0.1cm}
    \label{fig:percentage}
\end{figure*}
\begin{figure*}[ht]
    \centering
    \begin{minipage}{0.25\textwidth}
        \captionof{table}{Computation complexity of \acrshort{model}, \acrshort{gs}, and \acrshort{rd}.}
        \begin{tabular}{lc}
            \toprule
            \makecell{Approach\\Name} & \makecell{Computation\\Complexity} \\
            \midrule
            \acrshort{model} & \(O(1)\) \\
            \acrshort{gs} & \(O(n^2)\) \\
            \acrshort{rd} & \(O(kn)\) \\
            \bottomrule
        \end{tabular}
        \label{tab:computation_complexity}
    \end{minipage}
    \hfill
    \begin{minipage}{0.7\textwidth}
        \captionsetup[subfigure]{justification=centering}
        \begin{subfigure}{0.3\textwidth}
            \includegraphics[width=\textwidth]{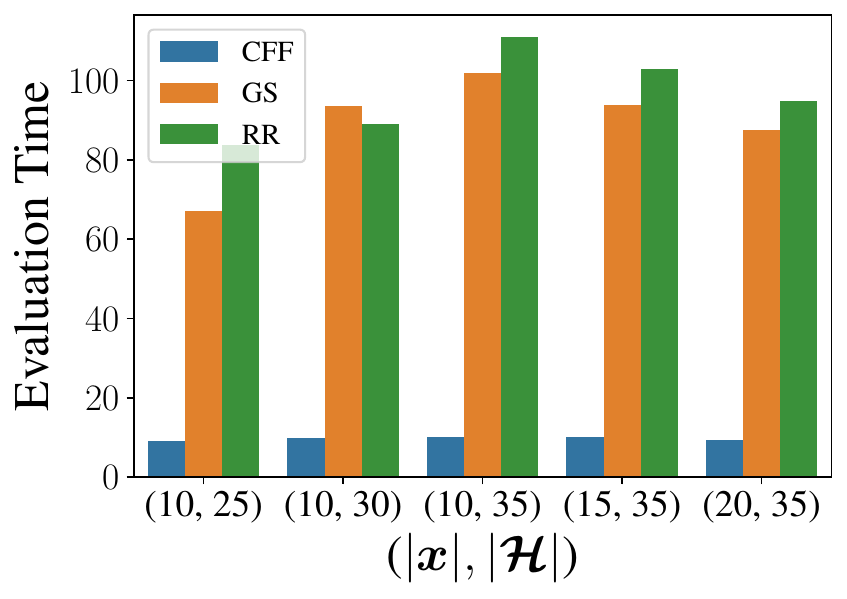}
        \end{subfigure}
        \begin{subfigure}{0.3\textwidth}
            \includegraphics[width=\textwidth]{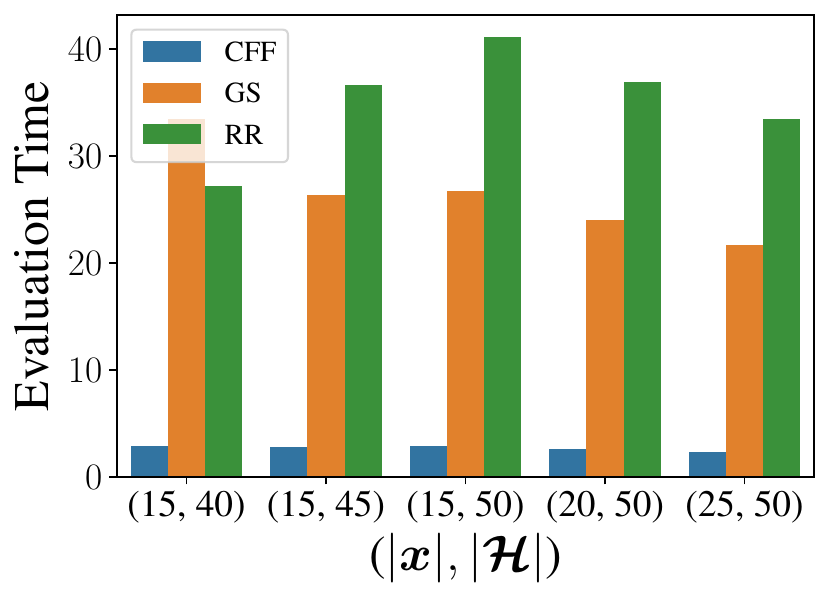}
        \end{subfigure}
        \begin{subfigure}{0.3\textwidth}
            \includegraphics[width=\textwidth]{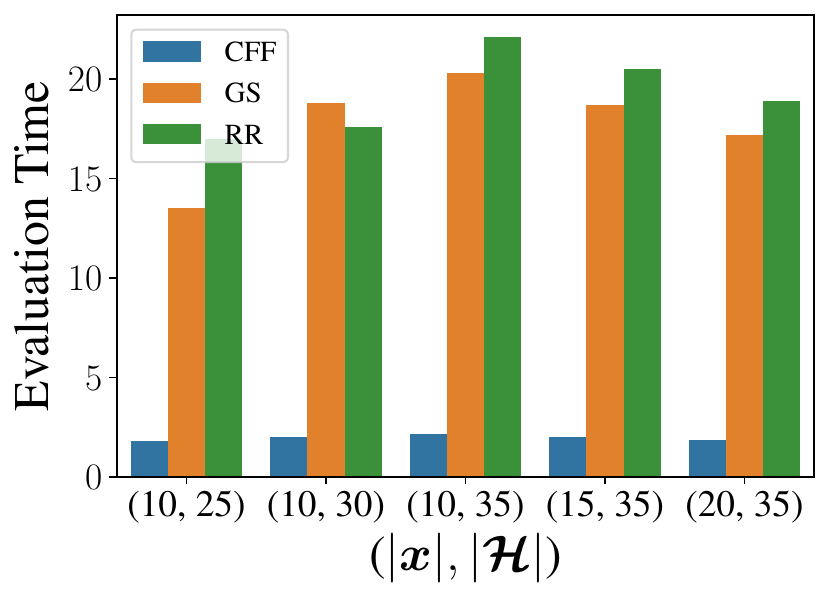}
        \end{subfigure}
    \caption{Total time used (in hours) by \acrshort{model} (Ours), \acrshort{gs}, and \acrshort{rd} to solve all \gls{ehd} tasks in test dataset on \gls{retweet}, \gls{so}, and \gls{yelp} (from left to right). Lower is better.}\label{fig:speed1}
    \vspace{0.25cm}
\end{minipage}
\end{figure*}

\mysubsection[Effectiveness of le and ln]{Effectiveness of \(\losse\) and \(\lossn\)}\label{sec:proof_l_c_and_lossn}

The parameters in \acrshort{model} are trained by minimizing loss \(\losse\) and \(\lossn\).
Minimizing \(\losse\) forces \acrshort{model} to move more events from \(\history_l\) (a copy of \(\history\) initially) to \(\history_{d}\) so that the two constraints in \gls{ehd} are satisfied. On the other hand, minimizing \(\lossn\) encourages \acrshort{model} to move fewer events to \(\history_{d}\) so that \(|\history_{d}|\) is minimized. 
To verify this, \cref{fig:efficiency_so} (a) reports the number of events in \(\history_d\) from the \acrshort{model} trained with \(\losse\) on \gls{so}, and (b) shows the same for the model trained with 
\(\lossn\).
As expected, all events are moved to \(\history_{d}\) in \cref{fig:efficiency_so} (a) while no event is moved to \(\history_{d}\) in \cref{fig:efficiency_so} (b). Similar results can be observed on other datasets in \cref{app:additional_efficiency}.



\mysubsection[Impact of \(\epsilon_l\) and \(\epsilon_d\) Settings]{Impact of \(\epsilon_l\) and \(\epsilon_d\) Settings}\label{sec:thresholds_influence}

The hyperparameters \(\epsilon_l\) and \(\epsilon_d\) in \cref{eqn:new_optimization_problem} give users control over the resulting \(\history_d\). \cref{tab:different_thresholds} shows \(|\history_d|\) under different \((\epsilon_l,\epsilon_d)\) combinations on \gls{so} where (\(|\history| = 50\), \(|\data{x}| = 15\)). For a given \(\epsilon_d=0.9\), the lower bound of prediction accuracy based on \(\history_d\) is fixed; a lower \(\epsilon_l\) forces the prediction accuracy based on \(\history_l\) to decrease, so \(\history_l\) tends to have fewer events, and \(\history_d\) tends to have more events. For a given \(\epsilon_l=0.5\), the upper bound of prediction accuracy based on \(\history_l\) is fixed; a higher \(\epsilon_d\) forces the prediction accuracy based on \(\history_d\) to increase, so \(\history_d\) tends to have more events and \(\history_l\) tends to have fewer events. The results in other datasets show a similar trend. In practice, we find that a high \(\epsilon_d\) like 0.9 with a reasonably low \(\epsilon_l\) like 0.5 or 0.6 leads to \(\history_{d}\) including a reasonably small number of events in history, but essential for an accurate output of \acrshort{mtpp}. This motivates the default setting \((\epsilon_l, \epsilon_d) = (0.5, 0.9)\) for \gls{retweet} and \gls{so} and \((\epsilon_l, \epsilon_d) = (0.6, 0.9)\) for \gls{yelp}.

\begin{table}
    \small
    \centering
    \caption{
        The length of \(\history_d\) obtained by \acrshort{model} under different \((\epsilon_l,\epsilon_d)\). 
        The results demonstrate the impact of \(\epsilon_l\) and \(\epsilon_d\) (\(|\history| = 50\), \(|\data{x}| = 15\)).
        }
    \begin{tabular}{cc|cc}
        \toprule
        (\(\epsilon_l\), \(\epsilon_d\)) & \(|\history_d|\) & (\(\epsilon_l\), \(\epsilon_d\)) & \(|\history_d|\) \\
        \midrule
        (0.01, 0.9) & \meanstd{49.994}{0.0001} & (0.5, 0.51) & \meanstd{14.806}{0.2635} \\
        (0.3, 0.9)  & \meanstd{32.083}{0.2620} & (0.5, 0.7)  & \meanstd{19.034}{0.2061} \\
        (0.4, 0.9)  & \meanstd{29.354}{0.2568} & (0.5, 0.8)  & \meanstd{22.111}{0.6095} \\
        (0.5, 0.9)  & \meanstd{26.115}{0.5226} & (0.5, 0.9)  & \meanstd{26.115}{0.5226} \\
        (0.6, 0.9)  & \meanstd{24.694}{0.1785} & (0.5, 0.99) & \meanstd{34.309}{0.7082} \\
        \bottomrule
    \end{tabular}
    \label{tab:different_thresholds}
    \vskip -0.1in
\end{table}

\mysubsection[Correctness Check of Hd]{Explanatory Capability of \(\history_d\)}\label{sec:case_study}

This section quantitatively and qualitatively evaluates the explanatory capability of \(\history_d\)s generated by \acrshort{model}. 
Specifically, we compare \(\history_d\)s generated by \acrshort{model} against the random subsequences in \(\history\) of the same length as \(\history_d\)s. It should be noted that random subsequences in \(\history\) of the same length as \(\history_d\)s are not required to satisfy the constraints in \cref{{eqn:new_optimization_problem}}, so they are not the solution of \gls{ehd} and the method to obtain them are not baselines.

\partopic{Quantitative Assessment: \acrshort{otd}~\citep{meiImputingMissingEvents2019,panosDecomposableTransformerPoint2024}}

We evaluate the explanatory capability of \(\history_d\) by employing the \acrshort{mtpp} model \(\mathcal{M}\) to predict the next \(|\data{x}|\) events based on \(\history_d\), denoted \(\data{x}^{\prime}\). Then, we measure \acrshort{otd} between \(\data{x}^{\prime}\) and \(\data{x}\). Lower \acrshort{otd} indicates \(\data{x}^{\prime}\) is more similar to \(\data{x}\), so \(\history_d\) is a better explanation. 
As a reference, we also measure \acrshort{otd} where (i) \(\data{x}^{\prime}\) is obtained based on a random subsequence of \(\history\) whose length is equal to \(\history_d\), and (ii) \(\data{x}^{\prime}\) is obtained based on full history \(\history\).
The OTD values are reported in \cref{tab:pred_performance}. We observe that \(\data{x}^{\prime}\)s based on \(\history_d\)s returned by \acrshort{model} are significantly more similar to \(\data{x}\) than those based on random subsequences, as indicated by the mean and standard deviation of \acrshort{otd}. This shows that \(\history_d\)s generated by \acrshort{model} have much more explanatory capability at their length. Interestingly, on some datasets, \eg \gls{yelp} with \((|\data{x}|, |\history|) = (10, 25)\), \(\model\) predicts \(\data{x}\) more accurately based on \(\history_d\) than \(\history\). A possible reason is that \acrshort{model} effectively removes noise from \(\history\) to form \(\history_d\).

\begin{table}[h]
    \small
    \centering
    \caption{The \acrshort{otd} between \(\data{x}\) and \(\data{x}^{\prime}\). \(\data{x}^{\prime}\) is based on \(\history_d\) returned by \acrshort{model}, random select, and full history \(\history\), respectively. Lower is better.}
    \begin{tabular}{lccccc}
    \toprule
    & \(\lvert\data{x}\rvert\) & \(\lvert\history\rvert\) & \acrshort{model} &  Random Select & \makecell{Full History\\(Reference)} \\
    \midrule
        \multirow{5}{*}{\rotatebox[origin=c]{90}{\scriptsize{\gls{so}}}}
                                       & 15 & 40 & \textbf{\meanstd{1.4999}{0.0045}}& \meanstd{1.5531}{0.0047} & \meanstd{1.4701}{0.0043}  \\
                                       & 15 & 45 & \textbf{\meanstd{1.4766}{0.0007}} & \meanstd{1.5381}{0.0028} & \meanstd{1.4341}{0.0004} \\
                                       & 15 & 50 & \textbf{\meanstd{1.4522}{0.0009}} & \meanstd{1.5215}{0.0008} & \meanstd{1.4078}{0.0008} \\
                                       & 20 & 50 & \textbf{\meanstd{1.4421}{0.0021}} & \meanstd{1.5248}{0.0023} & \meanstd{1.3839}{0.0034} \\
                                       & 25 & 50 & \textbf{\meanstd{1.4337}{0.0019}} & \meanstd{1.5270}{0.0028} & \meanstd{1.3757}{0.0035} \\
    \midrule
        \multirow{5}{*}{\rotatebox[origin=c]{90}{\gls{retweet}}} 
                                 & 10 & 25 & \textbf{\meanstd{1.9147}{0.0017}} & \meanstd{1.9333}{0.0016} & \meanstd{1.9115}{0.0005} \\
                                 & 10 & 30 & \textbf{\meanstd{1.9192}{0.0011}} & \meanstd{1.9393}{0.0028} & \meanstd{1.9167}{0.0012} \\
                                 & 10 & 35 & \textbf{\meanstd{1.9263}{0.0006}} & \meanstd{1.9504}{0.0010} & \meanstd{1.9247}{0.0013} \\
                                 & 15 & 35 & \textbf{\meanstd{1.9231}{0.0025}} & \meanstd{1.9434}{0.0021} & \meanstd{1.9206}{0.0009} \\
                                 & 20 & 35 & \textbf{\meanstd{1.9213}{0.0013}} & \meanstd{1.9309}{0.0006} & \meanstd{1.9200}{0.0003} \\
    \midrule
        \multirow{5}{*}{\rotatebox[origin=c]{90}{\gls{yelp}}} 
                             & 10 & 25 & \textbf{\meanstd{1.7362}{0.0006}} & \meanstd{1.7983}{0.0005} & \meanstd{1.7425}{0.0000} \\
                             & 10 & 30 & \textbf{\meanstd{1.7383}{0.0004}} & \meanstd{1.8052}{0.0004} & \meanstd{1.7382}{0.0009} \\
                             & 10 & 35 & \textbf{\meanstd{1.7674}{0.0005}} & \meanstd{1.8155}{0.0002} & \meanstd{1.7264}{0.0003} \\
                             & 15 & 35 & \textbf{\meanstd{1.7317}{0.0008}} & \meanstd{1.8035}{0.0005} & \meanstd{1.7120}{0.0004} \\
                             & 20 & 35 & \textbf{\meanstd{1.7055}{0.0002}} & \meanstd{1.7828}{0.0000} & \meanstd{1.6999}{0.0005} \\
    \bottomrule
    \end{tabular}
    \vskip -0.25in
    \label{tab:pred_performance}
\end{table}

\partopic{Quanlitative Assessment: Mark Percentage}
In the test data of \gls{retweet}, the \(\history_d\)s returned by \acrshort{model} versus those randomly selected are compared in terms of percentage of marks. Retweet logs the retweet activities of regular and famous users on Twitter. The percentage of marks is the ratio between events with one specific mark in \(\history_d\) and those in the corresponding \(\history\). \cref{fig:percentage} presents the mark percentage of \(\history_d\) returned by \acrshort{model} and that of \(\history_d\) selected randomly. For a fair comparison, the length of randomly selected \(\history_d\) is always equal to that of \(\history_d\) by \acrshort{model}. We find that mark 2, representing famous users, is consistently more frequent in \(\history_d\) by \acrshort{model} while other marks are not.
These results confirm that retweets from influential users drive future activity~\citep{aimeurFakeNewsDisinformation2023,baribi-bartovSupersharersFakeNews2024}, indicating that \(\history_d\) by CFF captures key factors. Such comparison is not possible on StackOverflow and Yelp, where marks lack meaningful semantics.

\mysubsection[Comparison with baselines]{Comparison with baselines}\label{sec:audc}
\partopic{Size of \(\history_d\)}

Under the two constraints in \cref{eqn:new_optimization_problem}, the resultant \(\history_d\) with fewer events indicates a better solution. \cref{tab:distill_length} reports the average of \(|\history_d|\) using \acrshort{model} and baselines. First, \acrshort{gs} outperforms \acrshort{rd} by a consistent and noticeable margin on all datasets. 
Second, our \acrshort{model} demonstrates the performance better than both baselines. The experimental results demonstrate that the problem is difficult, as we cannot properly solve it with a simple solution like \acrshort{rd}. The results also demonstrate that our \acrshort{model} works as expected. Looking closely, \acrshort{gs} repeatedly identifies the individual event that affects \(L_e\) the most and moves it from \(\history_l\) to \(\history_d\). This method cannot capture the effect of event combinations in history and may lead to suboptimal solutions. In contrast, our \acrshort{model} overcomes the weakness of \acrshort{gs} by searching for optimal event combinations, demonstrating better performance.


\begin{table}[h]
    \small
    \centering
    \caption{
    The average of \(|\history_d|\) using \acrshort{model} and baselines. The standard deviation of \acrshort{gs} is 0 because \acrshort{gs} is deterministic. 
    }
    \begin{tabular}{lccccc}
    \toprule
    & \(\lvert\data{x}\rvert\) & \(\lvert\history\rvert\) & \acrshort{model} & \acrshort{gs} & \acrshort{rd} \\
    \midrule
        \multirow{5}{*}{\rotatebox[origin=c]{90}{\scriptsize{\gls{so}}}}
                                       & 15 & 40 & \textbf{\meanstd{21.484}{0.0073}} & \meanstd{23.681}{0.0000} & \meanstd{36.424}{0.0033} \\
                                       & 15 & 45 & \textbf{\meanstd{23.700}{0.0802}} & \meanstd{25.700}{0.0000} & \meanstd{40.582}{0.0042} \\
                                       & 15 & 50 & \textbf{\meanstd{26.115}{0.5226}} & \meanstd{27.699}{0.0000} & \meanstd{44.693}{0.0015} \\
                                       & 20 & 50 & \textbf{\meanstd{27.416}{0.0974}} & \meanstd{28.927}{0.0000} & \meanstd{44.898}{0.0046} \\
                                       & 25 & 50 & \textbf{\meanstd{27.811}{0.2973}} & \meanstd{29.636}{0.0000} & \meanstd{45.159}{0.0011} \\
    \midrule
        \multirow{5}{*}{\rotatebox[origin=c]{90}{\gls{retweet}}} 
                                 & 10 & 25 & \textbf{\meanstd{12.281}{0.2001}} & \meanstd{14.722}{0.0000} & \meanstd{24.004}{0.0004} \\
                                 & 10 & 30 & \textbf{\meanstd{13.297}{0.2264}} & \meanstd{16.511}{0.0000} & \meanstd{28.620}{0.0003} \\
                                 & 10 & 35 & \textbf{\meanstd{14.390}{0.0899}} & \meanstd{18.053}{0.0000} & \meanstd{33.207}{0.0018} \\
                                 & 15 & 35 & \textbf{\meanstd{20.632}{0.5377}} & \meanstd{24.875}{0.0000} & \meanstd{34.532}{0.0008} \\
                                 & 20 & 35 & \textbf{\meanstd{28.140}{1.4211}} & \meanstd{29.990}{0.0000} & \meanstd{34.894}{0.0006} \\
    \midrule
        \multirow{5}{*}{\rotatebox[origin=c]{90}{\gls{yelp}}} 
                             & 10 & 25 & \textbf{\meanstd{9.6412}{0.0148}} & \meanstd{11.640}{0.0000} & \meanstd{23.112}{0.5788} \\
                             & 10 & 30 & \textbf{\meanstd{9.8174}{0.0898}} & \meanstd{12.587}{0.0000} & \meanstd{27.396}{0.7311} \\
                             & 10 & 35 & \textbf{\meanstd{10.008}{0.2310}} & \meanstd{13.508}{0.0000} & \meanstd{31.600}{0.9164} \\
                             & 15 & 35 & \textbf{\meanstd{13.422}{0.0436}} & \meanstd{18.237}{0.0000} & \meanstd{33.257}{0.6701} \\
                             & 20 & 35 & \textbf{\meanstd{18.160}{0.4387}} & \meanstd{22.562}{0.0000} & \meanstd{34.114}{0.4259} \\
    \bottomrule
    \end{tabular}
    \label{tab:distill_length}
    \vskip -0.1in
\end{table}



\partopic{Time Efficiency}\label{sec:speed}

Suppose the length of \(\history\) is \(n\). The computation complexity of \acrshort{model}, \acrshort{gs}, and \acrshort{rd} are reported in \cref{tab:computation_complexity}. In one iteration, \acrshort{gs} moves one event from \(\history_l\) to \(\history_d\) that increases \(\log\mathrm{ppl}(\data{x}|\history_l) - \log\mathrm{ppl}(\data{x}|\history_d)\) the most. \acrshort{gs} solves \gls{ehd} by running iterations until the gap satisfies the constraints in \cref{eqn:new_optimization_problem}. The computation complexity in the worst case is \(O(n^2)\). \acrshort{rd} randomly moves \(Q\) events from \(\history_l\) to \(\history_d\) for \(\log\mathrm{ppl}(\data{x}|\history) - \log\mathrm{ppl}(\data{x}|\history_l)\) and \(\log\mathrm{ppl}(\data{x}|\history) - \log\mathrm{ppl}(\data{x}|\history_d)\). This process runs \(k\) times in one iteration for the average of mentioned two gaps. \(Q\) starts from 1. \acrshort{rd} solves \gls{ehd} when the average gap satisfies the constraints in \cref{eqn:new_optimization_problem}, otherwise, we increase \(Q\) by 1 and rerun the iteration. The computation complexity of \acrshort{rd} in the worst case is \(O(kn)\). Our approach obtains \(\history_d\) and \(\history_l\) from \(p(\vec{y}|\history, \data{x})\), which is computed by \acrshort{model} in constant time. So The computation complexity of \acrshort{model} in the worst case is \(O(1)\).


\cref{fig:speed1} reports the total time of our \acrshort{model} and baselines to solve all \gls{ehd} in the test dataset.
The results show that \acrshort{model} is 6-10 times faster than baselines. \acrshort{gs} and \acrshort{rd} have to interact with the \acrshort{mtpp} model multiple times for one \(\history_d\). In contrast, the \acrshort{model} does not need to interact with \acrshort{mtpp} model because it already learned which events should be selected from history during training offline.

\mysection{Conclusions}

This work proposes investigating \gls{ehd}. We find that the obtained explanation is irrational if we define \gls{ehd} as \gls{ca} or \gls{fa}. To overcome the issue, this study proposes to define \gls{ehd} as a combination of \gls{ca} and \gls{fa}. As a combinatorial problem, the optimal solution of \gls{ehd} is intractable. We deliberately design a learning-based solution named \acrfull{model} by probing various combinations of historical events with techniques enabling backpropagation with the Gumbel-softmax trick and disclosing the consistent relation between the differentiable \(\norml{1}\) and non-differentiable \(\norml{0}\) in the context of our problem. The superiority of \acrshort{model} over baselines in terms of explanation quality and processing efficiency has been verified by extensive experiments.



\clearpage


\bibliography{reference.bib}

\begin{thebibliography}{62}
\providecommand{\natexlab}[1]{#1}
\providecommand{\url}[1]{\texttt{#1}}
\expandafter\ifx\csname urlstyle\endcsname\relax
  \providecommand{\doi}[1]{doi: #1}\else
  \providecommand{\doi}{doi: \begingroup \urlstyle{rm}\Url}\fi

\bibitem[Abrate and Bonchi(2021)]{abrate_counterfactual_2021}
C.~Abrate and F.~Bonchi.
\newblock {Counterfactual} {Graphs} for {Explainable} {Classification} of {Brain} {Networks}.
\newblock In \emph{{KDD}}, 2021.

\bibitem[Aimeur et~al.(2023)Aimeur, Amri, and Brassard]{aimeurFakeNewsDisinformation2023}
E.~Aimeur, S.~Amri, and G.~Brassard.
\newblock Fake news, disinformation and misinformation in social media: a review.
\newblock \emph{Social Network Analysis and Mining}, 2023.

\bibitem[Baribi-Bartov et~al.(2024)Baribi-Bartov, Swire-Thompson, and Grinberg]{baribi-bartovSupersharersFakeNews2024}
S.~Baribi-Bartov, B.~Swire-Thompson, and N.~Grinberg.
\newblock Supersharers of fake news on twitter.
\newblock \emph{Science}, 2024.

\bibitem[Barkan et~al.(2024)Barkan, Bogina, Gurevitch, Asher, and Koenigstein]{barkanCounterfactualFrameworkLearning2024}
O.~Barkan, V.~Bogina, L.~Gurevitch, Y.~Asher, and N.~Koenigstein.
\newblock A counterfactual framework for learning and evaluating explanations for recommender systems.
\newblock In \emph{{WWW}}, 2024.

\bibitem[Bengio et~al.(2013)Bengio, Leonard, and Courville]{bengio_estimating_2013}
Y.~Bengio, N.~Leonard, and A.~Courville.
\newblock Estimating or {Propagating} {Gradients} {Through} {Stochastic} {Neurons} for {Conditional} {Computation}.
\newblock In \emph{arXiv:1308.3432}, 2013.

\bibitem[Budhathoki et~al.(2021)Budhathoki, Janzing, Bloebaum, and Ng]{budhathokiWhyDidDistribution2021}
K.~Budhathoki, D.~Janzing, P.~Bloebaum, and H.~Ng.
\newblock Why did the distribution change?
\newblock In \emph{AISTATS}, 2021.

\bibitem[Cai et~al.(2025)Cai, Zhu, Chen, Fang, Wu, Qiao, and Hao]{caiProbabilityNecessitySufficiency2025}
R.~Cai, Y.~Zhu, X.~Chen, Y.~Fang, M.~Wu, J.~Qiao, and Z.~Hao.
\newblock On the probability of necessity and sufficiency of explaining graph neural networks: A lower bound optimization approach.
\newblock \emph{Neural Networks}, 2025.

\bibitem[Chen et~al.(2022)Chen, Silvestri, Wang, Zhang, Huang, Ahn, and Tolomei]{chen_grease_2022}
Z.~Chen, F.~Silvestri, J.~Wang, Y.~Zhang, Z.~Huang, H.~Ahn, and G.~Tolomei.
\newblock {GREASE}: {Generate} {Factual} and {Counterfactual} {Explanations} for {GNN}-based {Recommendations}.
\newblock In \emph{arXiv:2208.04222}, 2022.

\bibitem[Daley and Vere-Jones(2003)]{daley_introduction_2003}
D.~J. Daley and D.~Vere-Jones.
\newblock \emph{An {Introduction} to the {Theory} of {Point} {Processes} {Volume} {I}: {Elementary} {Theory} and {Methods}}.
\newblock Springer, 2003.

\bibitem[Enguehard et~al.(2020)Enguehard, Busbridge, Bozson, Woodcock, and Hammerla]{enguehardNeuralTemporalPoint2020}
J.~Enguehard, D.~Busbridge, A.~Bozson, C.~Woodcock, and N.~Hammerla.
\newblock {Neural} {Temporal} {Point} {Processes} for {Modelling} {Electronic} {Health} {Records}.
\newblock In \emph{The Machine Learning for Health {NeurIPS} Workshop}, 2020.

\bibitem[Fernandez et~al.(2022)Fernandez, Aledo, Gamez, and Puerta]{fernandezFactualCounterfactualExplanations2022}
G.~Fernandez, J.~A. Aledo, J.~A. Gamez, and J.~M. Puerta.
\newblock Factual and counterfactual explanations in fuzzy classification trees.
\newblock \emph{{IEEE} Transactions on Fuzzy Systems}, 2022.

\bibitem[Ferraro(2009)]{ferraro2009counterfactual}
P.~J. Ferraro.
\newblock {Counterfactual} {Thinking} and {Impact} {Evaluation} in {Environmental} {Policy}.
\newblock \emph{New directions for evaluation}, 2009.

\bibitem[Gao et~al.(2021)Gao, Subramanian, Bhattacharjya, Shou, Mattei, and Bennett]{gao2021causal}
T.~Gao, D.~Subramanian, D.~Bhattacharjya, X.~Shou, N.~Mattei, and K.~Bennett.
\newblock {Causal} {Inference} for {Event} {Pairs} in {Multivariate} {Point} {Processes}.
\newblock In \emph{NeurIPS}, 2021.

\bibitem[Gauch(2003)]{gauch2003scientific}
H.~G. Gauch.
\newblock \emph{Scientific method in practice}.
\newblock Cambridge University Press, 2003.

\bibitem[Ge et~al.(2024)Ge, Liu, Fu, Tan, Li, Xu, Li, Xian, and Zhang]{geSurveyTrustworthyRecommender2024}
Y.~Ge, S.~Liu, Z.~Fu, J.~Tan, Z.~Li, S.~Xu, Y.~Li, Y.~Xian, and Y.~Zhang.
\newblock A survey on trustworthy recommender systems.
\newblock \emph{{ACM} Trans. Recomm. Syst.}, 2024.

\bibitem[Ghazimatin et~al.(2020)Ghazimatin, Balalau, Roy, and Weikum]{ghazimatin_prince_2020}
A.~Ghazimatin, O.~Balalau, R.~S. Roy, and G.~Weikum.
\newblock {PRINCE:} {Provider}-side {Interpretability} with {Counterfactual} {Explanations} in {Recommender} {Systems}.
\newblock In \emph{{WSDM}}, 2020.

\bibitem[Goyal et~al.(2019)Goyal, Wu, Ernst, Batra, Parikh, and Lee]{goyal2019counterfactual}
Y.~Goyal, Z.~Wu, J.~Ernst, D.~Batra, D.~Parikh, and S.~Lee.
\newblock {Counterfactual} {Visual} {Explanations}.
\newblock In \emph{ICML}, 2019.

\bibitem[Guidotti(2022)]{guidotti_counterfactual_2022}
R.~Guidotti.
\newblock {Counterfactual} {Explanations} and {How} to {Find} {Them}: {Literature} {Review} and {Benchmarking}.
\newblock \emph{DMKD}, 2022.

\bibitem[Hizli et~al.(2023)Hizli, John, Juuti, Saarinen, Pietil{\"{a}}inen, and Marttinen]{Hizli2023}
{\c{C}}.~Hizli, S.~John, A.~Juuti, T.~Saarinen, K.~Pietil{\"{a}}inen, and P.~Marttinen.
\newblock Temporal {Causal} {Mediation} through a {Point} {Process}: {Direct} and {Indirect} {Effects} of {Healthcare} {Interventions}.
\newblock In \emph{NeurIPS}, 2023.

\bibitem[Huang and Chang(2023)]{huangReasoningLargeLanguage2023}
J.~Huang and K.~C.-C. Chang.
\newblock Towards reasoning in large language models: A survey.
\newblock In \emph{{ACL} 2023}, 2023.

\bibitem[Ide et~al.(2021)Ide, Kollias, Phan, and Abe]{ideCardinalityRegularizedHawkesGrangerModel2021}
T.~Ide, G.~Kollias, D.~T. Phan, and N.~Abe.
\newblock Cardinality-regularized hawkes-granger model.
\newblock In \emph{NeurIPS}, 2021.

\bibitem[Karlsson et~al.(2018)Karlsson, Rebane, Papapetrou, and Gionis]{karlssonExplainableTimeSeries2018}
I.~Karlsson, J.~Rebane, P.~Papapetrou, and A.~Gionis.
\newblock Explainable time series tweaking via irreversible and reversible temporal transformations.
\newblock In \emph{{ICDM}}, 2018.

\bibitem[Karp(1972)]{karp_reducibility_1972}
R.~M. Karp.
\newblock Reducibility among {Combinatorial} {Problems}.
\newblock In \emph{Complexity of {Computer} {Computations}}. 1972.

\bibitem[Lash et~al.(2017)Lash, Lin, Street, Robinson, and Ohlmann]{lash2017generalized}
M.~T. Lash, Q.~Lin, N.~Street, J.~G. Robinson, and J.~Ohlmann.
\newblock Generalized {Inverse} {Classification}.
\newblock In \emph{SIAM}, 2017.

\bibitem[Lei et~al.(2016)Lei, Barzilay, and Jaakkola]{lei_rationalizing_2016}
T.~Lei, R.~Barzilay, and T.~Jaakkola.
\newblock Rationalizing {Neural} {Predictions}.
\newblock In \emph{EMNLP}, 2016.

\bibitem[Leskovec and Krevl(2014)]{Leskovec2014SNAPD}
J.~Leskovec and A.~Krevl.
\newblock {SNAP Datasets}: {Stanford} {Large} {Network} {Dataset} {Collection}.
\newblock \url{http://snap.stanford.edu/data}, 2014.

\bibitem[Li et~al.(2022)Li, Feng, Wang, Essofi, Cao, Yan, and Song]{liExplainingPointProcesses2021}
S.~Li, M.~Feng, L.~Wang, A.~Essofi, Y.~Cao, J.~Yan, and L.~Song.
\newblock Explaining point processes by learning interpretable temporal logic rules.
\newblock In \emph{ICLR}, 2022.

\bibitem[Li et~al.(2023)Li, Chen, Xu, Ge, Tan, Liu, and Zhang]{liFairnessRecommendationFoundations2023}
Y.~Li, H.~Chen, S.~Xu, Y.~Ge, J.~Tan, S.~Liu, and Y.~Zhang.
\newblock Fairness in recommendation: Foundations, methods, and applications.
\newblock \emph{{ACM} Trans. Intell. Syst. Technol.}, 2023.

\bibitem[Lin et~al.(2021)Lin, Lan, and Li]{linGenerativeCausalExplanations2021}
W.~Lin, H.~Lan, and B.~Li.
\newblock Generative causal explanations for graph neural networks.
\newblock In \emph{ICML}, 2021.

\bibitem[Liu et~al.(2021)Liu, Chen, Liu, Zhang, and Xie]{liuMultiobjectiveExplanationsGNN2021}
Y.~Liu, C.~Chen, Y.~Liu, X.~Zhang, and S.~Xie.
\newblock Multi-objective explanations of {GNN} predictions.
\newblock In \emph{ICDM}, 2021.

\bibitem[Maddison et~al.(2017)Maddison, Mnih, and Teh]{maddison_concrete_2016}
C.~J. Maddison, A.~Mnih, and Y.~W. Teh.
\newblock The {Concrete} {Distribution}: {A} {Continuous} {Relaxation} of {Discrete} {Random} {Variables}.
\newblock In \emph{{ICLR}}, 2017.

\bibitem[Martens and Provost(2014)]{10.25300/MISQ/2014/38.1.04}
D.~Martens and F.~Provost.
\newblock Explaining {Data}-driven {Document} {Classifications}.
\newblock \emph{MIS Q.}, 2014.

\bibitem[Mei et~al.(2019)Mei, Qin, and Eisner]{meiImputingMissingEvents2019}
H.~Mei, G.~Qin, and J.~Eisner.
\newblock Imputing missing events in continuous-time event streams.
\newblock In \emph{ICML}, 2019.

\bibitem[Mothilal et~al.(2020)Mothilal, Sharma, and Tan]{mothilal_explaining_2020}
R.~K. Mothilal, A.~Sharma, and C.~Tan.
\newblock Explaining {Machine} {Learning} {Classifiers} through {Diverse} {Counterfactual} {Explanations}.
\newblock In \emph{{FAT}}, 2020.

\bibitem[Mu et~al.(2022)Mu, Li, Zhao, Wang, Ding, and Wen]{mu_alleviating_2022}
S.~Mu, Y.~Li, W.~X. Zhao, J.~Wang, B.~Ding, and J.-R. Wen.
\newblock Alleviating {Spurious} {Correlations} in {Knowledge}-aware {Recommendations} through {Counterfactual} {Generator}.
\newblock In \emph{{SIGIR}}, 2022.

\bibitem[Noorbakhsh and Rodriguez(2022)]{noorbakhsh_counterfactual_2022}
K.~Noorbakhsh and M.~G. Rodriguez.
\newblock Counterfactual {Temporal} {Point} {Processes}.
\newblock In \emph{NeurIPS}, 2022.

\bibitem[Omi et~al.(2019)Omi, Ueda, and Aihara]{omi_fully_2019}
T.~Omi, N.~Ueda, and K.~Aihara.
\newblock Fully {Neural} {Network} based {Model} for {General} {Temporal} {Point} {Processes}.
\newblock In \emph{NeurIPS}, 2019.

\bibitem[Panos(2024)]{panosDecomposableTransformerPoint2024}
A.~Panos.
\newblock Decomposable transformer point processes.
\newblock In \emph{NeurIPS}, 2024.

\bibitem[Parmentier and Vidal(2021)]{parmentier_optimal_2021}
A.~Parmentier and T.~Vidal.
\newblock Optimal {Counterfactual} {Explanations} in {Tree} {Ensembles}.
\newblock In \emph{{ICML}}, 2021.

\bibitem[Pearl(2009)]{pearl_causality_2009}
J.~Pearl.
\newblock \emph{Causality}.
\newblock Cambridge University Press, 2009.

\bibitem[Prado-Romero et~al.(2024)Prado-Romero, Prenkaj, Stilo, and Giannotti]{prado-romeroSurveyGraphCounterfactual2024}
M.~A. Prado-Romero, B.~Prenkaj, G.~Stilo, and F.~Giannotti.
\newblock A survey on graph counterfactual explanations: Definitions, methods, evaluation, and research challenges.
\newblock \emph{{ACM} Comput. Surv.}, 2024.

\bibitem[Prosperi et~al.(2020)Prosperi, Guo, Sperrin, Koopman, Min, He, Rich, Wang, Buchan, and Bian]{prosperi2020causal}
M.~Prosperi, Y.~Guo, M.~Sperrin, J.~S. Koopman, J.~S. Min, X.~He, S.~Rich, M.~Wang, I.~E. Buchan, and J.~Bian.
\newblock {Causal} {Inference} and {Counterfactual} {Prediction} in {Machine} {Learning} for {Actionable} {Healthcare}.
\newblock \emph{Nature Machine Intelligence}, 2\penalty0 (7):\penalty0 369--375, 2020.

\bibitem[Ramakrishnan et~al.(2020)Ramakrishnan, Lee, and Albarghouthi]{ramakrishnan_synthesizing_2020}
G.~Ramakrishnan, Y.~C. Lee, and A.~Albarghouthi.
\newblock {Synthesizing} {Action} {Sequences} for {Modifying} {Model} {Decisions}.
\newblock In \emph{{AAAI}}, 2020.

\bibitem[Sahoh and Choksuriwong(2022)]{sahohRoleExplainableArtificial2023}
B.~Sahoh and A.~Choksuriwong.
\newblock The {Role} of {Explainable} {Artificial} {Intelligence} in {High}-stakes {Decision}-making {Systems}: a {Systematic} {Review}.
\newblock \emph{J. Ambient Intell. Humanized Comput.}, 2022.

\bibitem[Schulam and Saria(2017)]{schulam2017reliable}
P.~Schulam and S.~Saria.
\newblock {Reliable} {Decision} {Support} using {Counterfactual} {Models}.
\newblock \emph{NeurIPS}, 2017.

\bibitem[Shchur et~al.(2021)Shchur, Türkmen, Januschowski, and Günnemann]{shchurNeuralTemporalPoint2021}
O.~Shchur, A.~C. Türkmen, T.~Januschowski, and S.~Günnemann.
\newblock Neural temporal point processes: A review.
\newblock In \emph{IJCAI}, 2021.

\bibitem[Song et~al.(2024)Song, Yang, Wang, An, and Li]{songLatentLogicTree2024}
Z.~Song, C.~Yang, C.~Wang, B.~An, and S.~Li.
\newblock Latent logic tree extraction for event sequence explanation from {LLMs}.
\newblock In \emph{{arXiv}:2406.01124}, 2024.

\bibitem[Tan et~al.(2021)Tan, Xu, Ge, Li, Chen, and Zhang]{tan_counterfactual_2021}
J.~Tan, S.~Xu, Y.~Ge, Y.~Li, X.~Chen, and Y.~Zhang.
\newblock Counterfactual {Explainable} {Recommendation}.
\newblock In \emph{CIKM}, 2021.

\bibitem[Tan et~al.(2022)Tan, Geng, Fu, Ge, Xu, Li, and Zhang]{tan_learning_2022}
J.~Tan, S.~Geng, Z.~Fu, Y.~Ge, S.~Xu, Y.~Li, and Y.~Zhang.
\newblock Learning and {Evaluating} {Graph} {Neural} {Network} {Explanations} based on {Counterfactual} and {Factual} {Reasoning}.
\newblock In \emph{{WWW}}, 2022.

\bibitem[Tran et~al.(2021)Tran, Ghazimatin, and Saha~Roy]{tran_counterfactual_2021}
K.~H. Tran, A.~Ghazimatin, and R.~Saha~Roy.
\newblock Counterfactual {Explanations} for {Neural} {Recommenders}.
\newblock In \emph{{SIGIR}}, 2021.

\bibitem[Verma et~al.(2020)Verma, Boonsanong, Hoang, Hines, Dickerson, and Shah]{verma_counterfactual_2022}
S.~Verma, V.~Boonsanong, M.~Hoang, K.~Hines, J.~Dickerson, and C.~Shah.
\newblock Counterfactual {Explanations} and {Algorithmic} {Recourses} for {Machine} {Learning}: A {Review}.
\newblock \emph{ACM Computing Surveys}, 2020.

\bibitem[Wu et~al.(2024)Wu, Ide, Kollias, Navratil, Lozano, Abe, Ma, and Yu]{wuLearningGrangerCausality2024}
D.~Wu, T.~Ide, G.~Kollias, J.~Navratil, A.~Lozano, N.~Abe, Y.~Ma, and R.~Yu.
\newblock Learning granger causality from instance-wise self-attentive hawkes processes.
\newblock In \emph{AISTATS}, 2024.

\bibitem[Xu et~al.(2022)Xu, Yu, Zhang, Ali, Ho, and Yang]{xu_counterfactual_2022}
R.~Xu, Y.~Yu, C.~Zhang, M.~K. Ali, J.~C. Ho, and C.~Yang.
\newblock Counterfactual and {Factual} {Reasoning} over {Hypergraphs} for {Interpretable} {Clinical} {Predictions} on {EHR}.
\newblock In \emph{MLHS}, 2022.

\bibitem[Xu et~al.(2024)Xu, Lamba, Ai, Tetreault, and Jaimes]{xuCFE2CounterfactualEditing2024a}
Z.~Xu, H.~Lamba, Q.~Ai, J.~Tetreault, and A.~Jaimes.
\newblock {CFE}2: Counterfactual editing for search result explanation.
\newblock In \emph{SIGIR}, 2024.

\bibitem[Yang et~al.(2024)Yang, Yang, Li, Fu, and Li]{yangNeuroSymbolicTemporalPoint2024}
Y.~Yang, C.~Yang, B.~Li, Y.~Fu, and S.~Li.
\newblock Neuro-symbolic temporal point processes, 2024.

\bibitem[Zhang et~al.(2020)Zhang, Lipani, Kirnap, and Yilmaz]{zhang_self-attentive_2020}
Q.~Zhang, A.~Lipani, {\"{O}}.~Kirnap, and E.~Yilmaz.
\newblock Self-attentive {Hawkes} {Process}.
\newblock In \emph{{ICML}}, 2020.

\bibitem[Zhang et~al.(2021{\natexlab{a}})Zhang, Lipani, and Yilmaz]{zhangLearningNeuralPoint2021}
Q.~Zhang, A.~Lipani, and E.~Yilmaz.
\newblock Learning neural point processes with latent graphs.
\newblock In \emph{WWW}, 2021{\natexlab{a}}.

\bibitem[Zhang et~al.(2023)Zhang, Zhang, Song, Adeshina, Zheng, Faloutsos, and Sun]{zhang_page-link_2023}
S.~Zhang, J.~Zhang, X.~Song, S.~Adeshina, D.~Zheng, C.~Faloutsos, and Y.~Sun.
\newblock {PaGE}-{Link}: {Path}-based {Graph} {Neural} {Network} {Explanation} for {Heterogeneous} {Link} {Prediction}.
\newblock In \emph{{WWW}}, 2023.

\bibitem[Zhang et~al.(2021{\natexlab{b}})Zhang, Sharma, and Liu]{zhang_vigdet_2021}
Y.~Zhang, K.~Sharma, and Y.~Liu.
\newblock {VigDet}: {Knowledge} {Informed} {Neural} {Temporal} {Point} {Process} for {Coordination} {Detection} on {Social} {Media}.
\newblock In \emph{NeurIPS}, 2021{\natexlab{b}}.

\bibitem[Zhang et~al.(2022)Zhang, Cao, and Liu]{zhang_counterfactual_2022}
Y.~Zhang, D.~Cao, and Y.~Liu.
\newblock Counterfactual {Neural} {Temporal} {Point} {Process} for {Estimating} {Causal} {Influence} of {Misinformation} on {Social} {Media}.
\newblock In \emph{NeurIPS}, 2022.

\bibitem[Zhao et~al.(2024)Zhao, Chen, Yang, Liu, Deng, Cai, Wang, Yin, and Du]{zhaoExplainabilityLargeLanguage2024}
H.~Zhao, H.~Chen, F.~Yang, N.~Liu, H.~Deng, H.~Cai, S.~Wang, D.~Yin, and M.~Du.
\newblock Explainability for large language models: A survey.
\newblock \emph{{ACM} Trans. Intell. Syst. Technol.}, 2024.

\bibitem[Zhao et~al.(2015)Zhao, Erdogdu, He, Rajaraman, and Leskovec]{zhao_seismic_2015}
Q.~Zhao, M.~A. Erdogdu, H.~Y. He, A.~Rajaraman, and J.~Leskovec.
\newblock {SEISMIC}: {A} {Self}-{Exciting} {Point} {Process} {Model} for {Predicting} {Tweet} {Popularity}.
\newblock In \emph{{SIGKDD}}, 2015.

\end{thebibliography}

\clearpage
\appendix

\mysection{Experiment Details}\label{app:model_training}

\mysubsection{MTPP Model}\label{app:mtpp_model}
\acrshort{model} can work with any \acrshort{mtpp} models that provide \(p^*(m, t)\). Without loss of generality, this study uses FullyNN~\citep{omi_fully_2019}. \cref{tab:fullynn_hyperparameters} presents the hyperparameters used for training the FullyNN on Retweet, StackOverflow, and Yelp.

FullyNN estimates the integral of conditional intensity functions \(\Lambda^*(m, t) = \int_{t_l}^{t}{\lambda^*(m, \tau)d\tau}\) and calculates the value of the intensity function at time \(t\) from the gradient of \(\Lambda^*(m, t)\):

\begin{align}
    \Lambda^*(m, t) &= \int_{t_l}^{t}{\lambda^*(m, \tau)d\tau} = \mathrm{FullyNN}(m, t) \\
    \lambda^*(m, t) &= \frac{\partial \Lambda^*(m, t)}{\partial t} = \frac{\partial \mathrm{FullyNN}(m, t)}{\partial t} \\
    p^*(m, t) & = \lambda^*(m, t)\exp{-\Lambda^*(t)} \\
    &= \frac{\partial \mathrm{FullyNN}(m, t)}{\partial t}\exp{-\sum_{n \in \set{M}}{\mathrm{FullyNN}(n, t)}}
\end{align}
This helps FullyNN elude calculating \(\Lambda^*(m, t)\) by numerical integration methods, such as Monte Carlo integration, to predict \acrshort{mtpp} faster and more accurately. The FullyNN is trained on NVIDIA A100 GPUs.


\begin{table}[!ht]
    \small
    \centering
    \caption{Hyperparamters settings for training \acrshort{mtpp} Models.}
    \begin{tabular}{l
        S[table-format=6.3]
        S[table-format=6.3]
        S[table-format=6.3]
        }
    \toprule
                     & {Retweet}  & {StackOverflow} & {Yelp} \\
    \midrule
    {\#Training Steps}   & 400000   & 200000 & 200000   \\
    {\#Warmup Steps}     & 80000    & 40000  & 40000    \\
    Batch Size       & 32       & 32     & 32       \\
    History Embedding& 32       & 32     & 32       \\
    Optimizer        & {AdamW}  & {AdamW}& {AdamW}  \\
    Intensity Vector & 16       & 32     & 16       \\
    Learning Rate    & 0.002    & 0.002  & 0.002    \\
    Layers           & 4        & 2      & 4        \\
    \bottomrule
    \end{tabular}
    \label{tab:fullynn_hyperparameters}
\end{table}

\mysubsection{Datasets}
\label{app:datasets}



We train and evaluate \acrshort{model} against baselines on three real-world datasets.

\textit{\gls{retweet}}~\citep{zhao_seismic_2015} records when users retweet a particular message on Twitter. The mark of this dataset distinguishes all users into three types. Mark 0 refers to the normal user, whose follower count is lower than the overall median. Mark 1 refers to the influential user, whose follower count is higher than the median but lower than the top-\num{5}{\%} of the entire user base. Mark 2 refers to the famous user, whose follower count is in the top-\num{5}{\%} of the entire user base. 

\textit{\gls{so}}~\citep{Leskovec2014SNAPD} was collected from Stackoverflow\footnote{\url{https://stackoverflow.com}}, a question-answering website. Users providing decent answers will receive different badges as rewards. We have 22 marks in this dataset, representing 22 different badges that users can receive for their answers.

\textit{\gls{yelp}}\footnote{\url{https://www.yelp.com}} contains the reviews of restaurants, shopping centers, and stores in the US on Yelp.
We categorize these reviews into three groups based on the reviewers. Mark 0 refers to the normal reviewer. The number of reviews a normal reviewer has is lower than the overall median, which is 5 reviews in our case. Mark 1 refers to the influential reviewers. These reviewers write more reviews than normal reviewers but less than the top-\num{5}{\%} reviewers. Mark 2 refers to the famous reviewers, the top-\num{5}{\%} reviewers who write more than 92 reviews.

\cref{tab:dataset_size} reports the number of events in training, validation, and test datasets for different settings of \(|\history|\) and \(|\data{x}|\) on dataset \gls{retweet}, \gls{so}, and \gls{yelp}. \cref{tab:ehd_hyperparameters} presents the hyperparameters used for training the \acrshort{model}. 

\begin{table}[!ht]
    \small
    \centering
    \caption{The number of events in training, validation, and test dataset for different settings of \(|\history|\) and \(|\data{x}|\).}
    \begin{tabular}{lc
        S[table-format = 7]
        S[table-format = 6]
        S[table-format = 6]
        }
        \toprule
                     &  (\(\data{x}\),\(\history\)) & {training} & {validation} & {test} \\
        \midrule
        \multirow{5}{*}{Retweet} & (10, 25)  & 1476116 & 145521 & 148465  \\
                                 & (10, 30)  & 1376116 & 135521 & 135521  \\
                                 & (10, 35)  & 1276116 & 125521 & 128465  \\
                                 & (15, 35)  & 1176383 & 115551 & 118497  \\
                                 & (20, 35)  & 1081289 & 106047 & 108970  \\
        \midrule
        \multirow{5}{*}{StackOverflow} & (15, 40) & 99791 & 10826 & 29232 \\
                                       & (15, 45) & 87623 & 9451 & 25824  \\
                                       & (15, 50) & 77341 & 8307 & 22951  \\
                                       & (20, 50) & 68635 & 7350 & 20504  \\
                                       & (25, 50) & 61254 & 6512 & 18385  \\
        \midrule
        \multirow{5}{*}{Yelp} & (10, 25) & 213677 & 25937 & 29562 \\
                              & (10, 30) & 197622 & 23952 & 27492  \\
                              & (10, 35) & 181567 & 21967 & 25422  \\
                              & (15, 35) & 165587 & 19996 & 23359  \\
                              & (20, 35) & 150640 & 18157 & 21406  \\
        \bottomrule
    \end{tabular}
    \label{tab:dataset_size}
\end{table}


\begin{table}[!ht]
    \small
    \centering
    \caption{Hyperparamters settings for training \acrshort{model}.}
    \begin{tabular}{l
        S[table-format=6.3]
        S[table-format=6.3]
        S[table-format=6.3]
        }
    \toprule
                      & {Retweet}  & {StackOverflow} & {Yelp} \\
    \midrule
    {\#Training Steps}& 100000          & 100000          & 100000          \\
    {\#Warmup Steps}  & 5000            & 5000            & 5000            \\
        Batch Size    & 256             & 128             & 128             \\
        Hidden Vector & 64              & 64              & 64              \\
        Input Vector  & 32              & 32              & 32              \\
          Q, K, V     & 32              & 32              & 32              \\
          Head        & 4               & 4               & 4               \\
          N           & 4               & 4               & 4               \\
          M           & 4               & 4               & 4               \\
        Learning Rate & 0.001           & 0.001           & 0.001           \\
        \(\epsilon_l\)& 0.5             & 0.5             & 0.6             \\
        \(\epsilon_d\)& 0.9             & 0.9             & 0.9             \\
        \(\alpha\)    & 1.0             & 1.0             & 1.0             \\
        \(\beta\)     & 1.0             & 1.0             & 1.0             \\
    \bottomrule
    \end{tabular}
    \label{tab:ehd_hyperparameters}
\end{table}

\mysection{Additional Experiment Results}


\mysubsection[Effectiveness of losse and lossn]{Effectiveness of CE-MTPP - Effectiveness of \(\losse\) and \(\lossn\)}\label{app:additional_efficiency}
\cref{sec:proof_l_c_and_lossn} demonstrate that minimizing loss \(\losse\) leads to \(\history_d\) with fewer events and minimizing loss \(L_n\) leads to \(\history_d\) with more events, respectively, on \gls{so}. The results on \gls{retweet} and \gls{yelp} are reported in \cref{fig:efficiency_retweet} and \cref{fig:efficiency_yelp}, respectively. They are consistent with the results in \cref{sec:proof_l_c_and_lossn}. 




\begin{figure*}[!ht]
    \centering
    \begin{subfigure}{\textwidth}
        \begin{subfigure}{0.2\textwidth}
            \includegraphics[width=\textwidth]{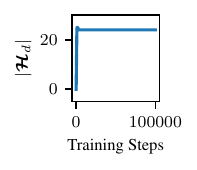}
        \end{subfigure}\hspace{-0.25cm}
        \begin{subfigure}{0.2\textwidth}
            \includegraphics[width=\textwidth]{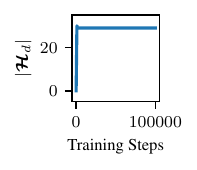}
        \end{subfigure}\hspace{-0.25cm}
        \begin{subfigure}{0.2\textwidth}
            \includegraphics[width=\textwidth]{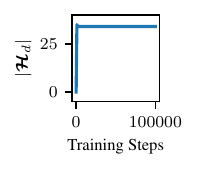}
        \end{subfigure}\hspace{-0.25cm}
        \begin{subfigure}{0.2\textwidth}
            \includegraphics[width=\textwidth]{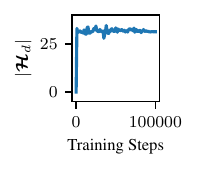}
        \end{subfigure}\hspace{-0.25cm}
        \begin{subfigure}{0.2\textwidth}
            \includegraphics[width=\textwidth]{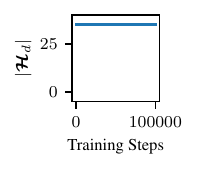}
        \end{subfigure}\hspace{-0.25cm}
        \caption{The number of event in \(\mathcal{H}_{d}\) returned by \acrshort{model} trained by minimizing \(\losse\) only.} 
        \vspace{0.5cm}
        \label{fig:l_te_retweet}
    \end{subfigure}
    \begin{subfigure}{\textwidth}
        \begin{subfigure}{0.2\textwidth}
            \includegraphics[width=\textwidth]{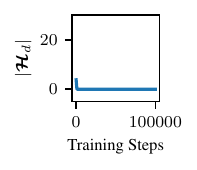}
        \end{subfigure}\hspace{-0.25cm}
        \begin{subfigure}{0.2\textwidth}
            \includegraphics[width=\textwidth]{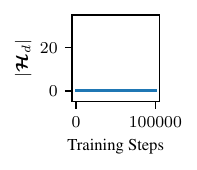}
        \end{subfigure}\hspace{-0.25cm}
        \begin{subfigure}{0.2\textwidth}
            \includegraphics[width=\textwidth]{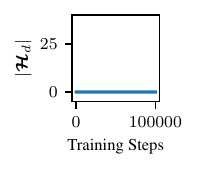}
        \end{subfigure}\hspace{-0.25cm}
        \begin{subfigure}{0.2\textwidth}
            \includegraphics[width=\textwidth]{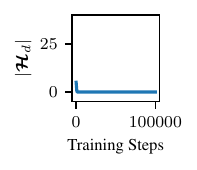}
        \end{subfigure}\hspace{-0.25cm}
        \begin{subfigure}{0.2\textwidth}
            \includegraphics[width=\textwidth]{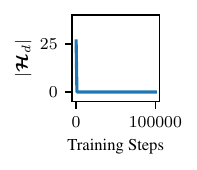}
        \end{subfigure}
        \vspace{0.5cm}
        \caption{The number of event in \(\mathcal{H}_{d}\) returned by \acrshort{model} trained by minimizing \(\lossn\) only.} 
        \label{fig:l_n_retweet}
    \end{subfigure}
    \vspace{0.5cm}
    \caption{Effectiveness of \(\losse\) and \(\lossn\) on \gls{retweet} (from left to right: \((|\data{x}|, |\history|)\) \(=\) \((15, 40),\) \((15, 45),\) \((15, 50),\) \((20, 50),\) \((25, 50)\)).}
    
    \label{fig:efficiency_retweet}
\end{figure*}

\begin{figure*}[!ht]
    \centering
    \begin{subfigure}{\textwidth}
        \begin{subfigure}{0.2\textwidth}
            \includegraphics[width=\textwidth]{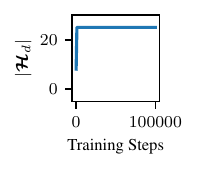}
        \end{subfigure}\hspace{-0.25cm}
        \begin{subfigure}{0.2\textwidth}
            \includegraphics[width=\textwidth]{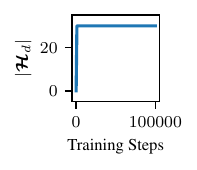}
        \end{subfigure}\hspace{-0.25cm}
        \begin{subfigure}{0.2\textwidth}
            \includegraphics[width=\textwidth]{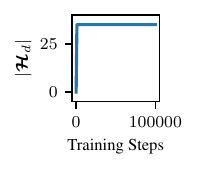}
        \end{subfigure}\hspace{-0.25cm}
        \begin{subfigure}{0.2\textwidth}
            \includegraphics[width=\textwidth]{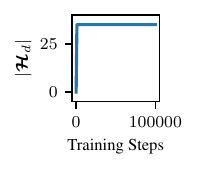}
        \end{subfigure}\hspace{-0.25cm}
        \begin{subfigure}{0.2\textwidth}
            \includegraphics[width=\textwidth]{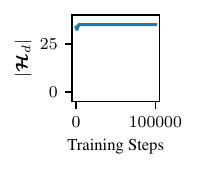}
        \end{subfigure}\hspace{-0.25cm}
        \caption{The number of event in \(\mathcal{H}_{d}\) returned by \acrshort{model} trained by minimizing \(\losse\) only.} 
        \vspace{0.5cm}
        \label{fig:l_te_yelp}
    \end{subfigure}
    \begin{subfigure}{\textwidth}
        \begin{subfigure}{0.2\textwidth}
            \includegraphics[width=\textwidth]{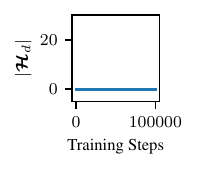}
        \end{subfigure}\hspace{-0.25cm}
        \begin{subfigure}{0.2\textwidth}
            \includegraphics[width=\textwidth]{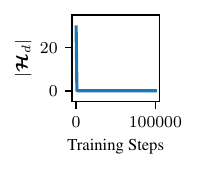}
        \end{subfigure}\hspace{-0.25cm}
        \begin{subfigure}{0.2\textwidth}
            \includegraphics[width=\textwidth]{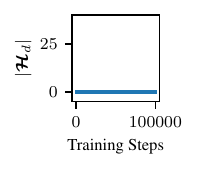}
        \end{subfigure}\hspace{-0.25cm}
        \begin{subfigure}{0.2\textwidth}
            \includegraphics[width=\textwidth]{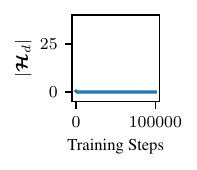}
        \end{subfigure}\hspace{-0.25cm}
        \begin{subfigure}{0.2\textwidth}
            \includegraphics[width=\textwidth]{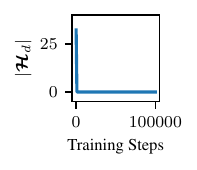}
        \end{subfigure}
        \caption{The number of event in \(\mathcal{H}_{d}\) returned by \acrshort{model} trained by minimizing \(\lossn\) only.} 
        \label{fig:l_n_yelp}
    \end{subfigure}
    \vspace{0.5cm}
    \caption{Effectiveness of \(\losse\) and \(\lossn\) on \gls{yelp} (from left to right: \((|\data{x}|, |\history|)\) \(=\) \((15, 40),\) \((15, 45),\) \((15, 50),\) \((20, 50),\) \((25, 50)\)).}
    \label{fig:efficiency_yelp}
\end{figure*}

\end{document}